\documentclass [twocolumn] {svjour3}	%

\usepackage[utf8]{inputenc} %
\usepackage{xcolor} %
\usepackage{newtxmath} %
\usepackage{amsmath} %
\usepackage{dsfont}
\usepackage[numbers]{natbib} %
\usepackage{hyperref} %
\usepackage{graphicx}
\usepackage{algorithm,algpseudocode} %
\usepackage{tikz}
\usepackage[export]{adjustbox}
\usepackage{subcaption}
\usepackage{calculator}
\usepackage{xspace}
\usepackage{overpic} %
\usepackage{siunitx}

\title{Ground Metric Learning on Graphs}

\author{Matthieu Heitz 	\and
	Nicolas Bonneel \and
	David Coeurjolly \and
	Marco Cuturi \and
	Gabriel Peyré
}

\institute{
	Matthieu Heitz
	\at {Université de Lyon, CNRS/LIRIS, France}
	\\\email{matthieu.heitz@univ-lyon1.fr}
	\and
	Nicolas Bonneel
	\at {Université de Lyon, CNRS/LIRIS, France}
	\and
	David Coeurjolly
	\at {Université de Lyon, CNRS/LIRIS, France}
	\and
	Marco Cuturi
	\at {Google Research, Brain team, France}
	\and 
	Gabriel Peyré
	\at {CNRS and ENS, PSL University, France}
}

\newcommand{\eqdef}{\stackrel{\text{\tiny def.}}{=}} %
\newcommand{\mat}[1]{\mathbf{#1}}
\newcommand{\sclr}[1]{#1}			%
\newcommand{\norm}[1]{\left|\left|#1\right|\right|}

\newcommand*{\ie}{i.e.\@\xspace}
\makeatletter
\newcommand*{\etc}{%
    \@ifnextchar{.}%
        {etc}%
        {etc.\@\xspace}%
}
\makeatother

\DeclareMathOperator*{\Id}{\mat{Id}}
\DeclareMathOperator*{\argmin}{argmin}
\DeclareMathOperator*{\KL}{KL}
\DeclareMathOperator*{\diag}{diag}

\newcommand{\highlightchanges}{0}
\newcommand{\new}[1]{\ifnum \highlightchanges = 1 {\color{blue}#1} \else #1 \fi}

\renewcommand{\vec}[1]{\mat{#1}}

\DeclareMathAlphabet{\mathcal}{OMS}{cmsy}{m}{n}

\begin{document}
\maketitle

\sloppy

\begin{abstract}
Optimal transport (OT) distances between probability distributions are parameterized by the ground metric they use between observations. 
Their relevance for real-life applications strongly hinges on whether that ground metric parameter is suitably chosen. The challenge of selecting it adaptively and algorithmically from prior knowledge, the so-called ground metric learning (GML) problem, has therefore appeared in various settings.
In this paper, we consider the GML problem when the learned metric is constrained to be a \emph{geodesic distance} on a \emph{graph} that supports the measures of interest.
This imposes a rich structure for candidate metrics, but also enables far more efficient learning procedures when compared to a direct optimization over the space of all metric matrices.
We use this setting to tackle an inverse problem stemming from the observation of a density evolving with time; we seek a graph ground metric such that the OT interpolation between the starting and ending densities that result from that ground metric agrees with the observed evolution.
This OT dynamic framework is relevant to model natural phenomena exhibiting displacements of mass, such as the evolution of the color palette induced by the modification of lighting and materials.


\end{abstract}

\section{Introduction}

Optimal transport (OT) is a powerful tool to compare probability measures supported on geometric domains (such as Euclidean spaces, surfaces or graphs).
The value provided by OT lies in its ability to leverage prior knowledge on the proximity of two isolated observations to quantify the discrepancy between two probability distributions of such observations.
This prior knowledge is usually encoded as a ``ground metric'' \cite{rubner_earth_2000}, which defines the cost of moving mass between points.

The Wasserstein distance between histograms, densities or point clouds, all seen here as particular instances of probability measures, is defined as the smallest cost required to transport one measure to another.
Because this distance is geodesic when the ground metric is geodesic, OT can also be used to compute interpolations between two probability measures, namely a path in the probability simplex that connects these two measures as end-points.
This interpolation is usually referred to as a \emph{displacement} interpolation \cite{mccann_convexity_1997}, describing a series of intermediate measures during the transport process.

When two discrete probability distributions are supported on a Euclidean space, and the ground metric is itself the Euclidean distance (the most widely used setting in applications), theory tells us that the displacement interpolation between these two measures only involves particles moving along straight lines, from one point in the starting measure to another in the end measure.
Imagine that, on the contrary, we observe a time series of measures in which mass displacements do not seem to match that hypothesis.
In that case the ground metric inducing such mass displacements must be of a different nature.
We cast in that case the following inverse problem: under which ground metric could this observed mass displacement be considered optimal?
The goal of our approach here is precisely to answer that question.
We give an illustrative example in \autoref{fig:main}, where we show that we search for a ground metric that deforms the space such that the sequence of mass displacements that is observed is close to a Wasserstein geodesic with that ground metric.
\begin{figure}
	\centering
	\includegraphics[width=\linewidth]{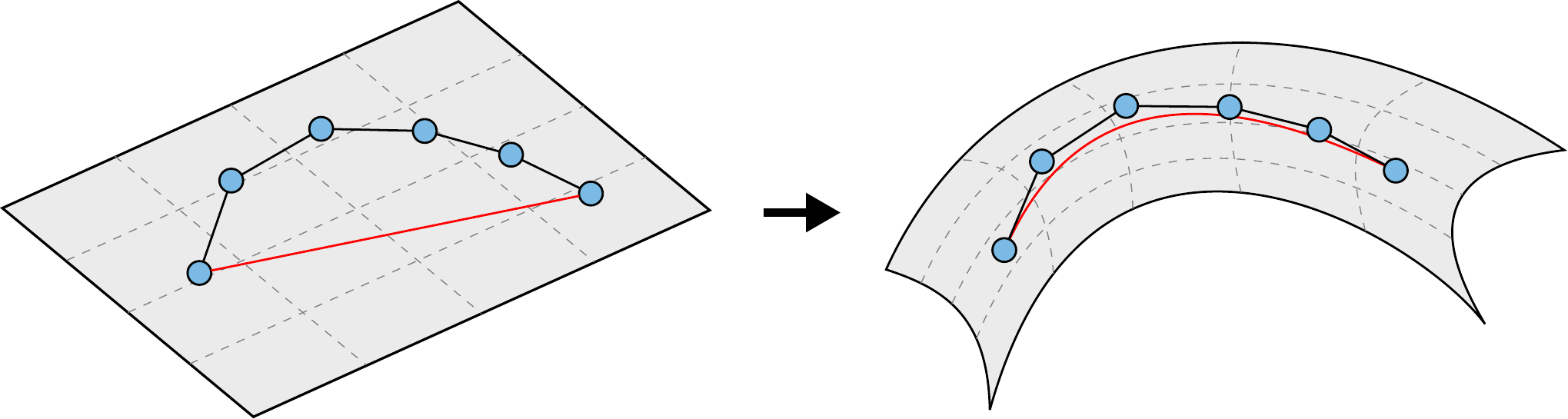}
	\caption{Left: before metric learning, the sequence of observed histograms (blue points) lies in the Wasserstein space of probability distributions with a Euclidean ground metric.
	The observed sequence does not match the Wasserstein geodesic (red line) between the first and last element.
	Right: after modifying the ground metric iteratively, the Wasserstein space is now deformed in such a way that the geodesic between the first and last element in this new geodesic space (red curve) is as close as possible to the sequence.}
	\label{fig:main}
\end{figure}

The main choice in our approach relies on looking at (anisotropic) diffusion-based geodesic distances \cite{yang_geodesic_2016} as the space of candidate ground metrics.
We then minimize the reconstruction error between measures that are observed at intermediary time stamps and interpolated histograms with that ground metric.
The problem we tackle is challenging in terms of time and memory complexity, due to repeated calls to solve Wasserstein barycenter problems with a non-Euclidean metric.
We address these issues using a sparse resolution of an anisotropic diffusion equation, yielding a tractable algorithm.
The optimization is performed using a quasi-Newton solver and automatic differentiation to compute the Jacobians of Wasserstein barycenters, here computed with entropic regularization and through a direct differentiation of Sinkhorn iterations~\cite{bonneel_wasserstein_2016,genevay_learning_2017}.
Because an automatic differentation of this entire pipeline would suffer from a prohibitive memory footprint, we also propose closed-form gradient formulas for the diffusion process.
We validate our algorithm on synthetic datasets, and on the learning of color variations in image sequences.
Finally, a Python implementation of our method as well as some datasets used in this paper are available online
\footnote{\href{https://github.com/matthieuheitz/2020-JMIV-ground-metric-learning-graphs}{https://github.com/matthieuheitz/2020-JMIV-ground-metric-learning-graphs}}.

\paragraph{\textbf{Contributions}}

\begin{itemize}
\renewcommand\labelitemi{$\bullet$}
	\item We introduce a new framework to learn the ground metric of optimal transport, where it is restricted to be a geodesic distance on a graph. The metric is parameterized as weights on the graph's edges, and geodesic distances are computed through anisotropic diffusion.

	\item We estimate this metric by fitting measures that are intermediate snapshots of a dynamical evolution of mass as Wasserstein barycenters.

	\item We provide a tractable algorithm based on the sparse discretization of the diffusion equation and efficient automatic differentiation. %

\end{itemize}

\section{Related Works}

\subsection{Computational Optimal transport}
\label{subsec:compote}

Solving OT problems has remained intractable for many years because doing so relies on solving a bipartite minimum-cost flow problem, with a number of variables that is quadratic with regard to the histograms' size.
Fortunately, in the past decade, methods to approximate OT distances using various types of regularizations have become widespread.
\citet{cuturi_sinkhorn_2013} introduced an entropic regularization of the problem, which allows the efficient approximation of the Wasserstein distance, using an iterative scaling method called the Sinkhorn algorithm.
This algorithm is very simple as it only performs point-wise operations on vectors, and matrix-vector multiplications that involve a kernel, defined as the exponential of minus the ground metric, inversely scaled by the regularization strength.
\citet{cuturi_fast_2014} then extended this method to compute Wasserstein barycenters, a concept which was previously introduced in \cite{agueh_barycenters_2011}.
\citet{benamou_iterative_2015} later linked this iterative scheme to Bregman projections, and showed that it can be adapted to solve various OT related problems such as partial, multi-marginal, or capacity-constrained OT.
This regularization allows the computation of OT for large problems, such as those arising in machine learning \cite{courty_optimal_2016,frogner_learning_2015,genevay_learning_2017} or computer graphics \cite{bonneel_wasserstein_2016,schmitz_wasserstein_2018,solomon_convolutional_2015}.

Recently, \citet{altschuler_massively_2018} introduced a method to accelerate the Sinkhorn algorithm via low-rank (Nyström) approximations of the kernel~\cite{altschuler_massively_2018}.
Simultaneously, there have been considerable efforts to study the convergence and approximation properties of the Sinkhorn algorithm~\cite{altschuler_near-linear_2017} and its variants \cite{dvurechensky_computational_2018}.

Other families of numerical methods are based on variational or PDE formulations of the problem \cite{papadakis_optimal_2014,angenent_minimizing_2003}, or semi-discrete formulations \cite{levy_numerical_2015}.
We refer to \cite{peyre_computational_2018} and \cite{santambrogio_optimal_2015} for extensive surveys on computational optimal transport.

The entropic regularization scheme has helped to tackle inverse problems that involve OT, since it converts the original Wasserstein distance into a fast, smooth, differentiable, and more robust loss.
Although differentiating the Wasserstein loss has been extensively covered, differentiating quantities that build upon it, such as smooth Wasserstein barycenters, is less common.
A few examples are Wasserstein barycentric coordinates~\cite{bonneel_wasserstein_2016}, Wasserstein dictionary learning on images~\cite{schmitz_wasserstein_2018} and graphs~\cite{simou_graph_2018}, and model ensembling~\cite{dognin_wasserstein_2019}.

\subsection{Metric Learning}

In machine learning, metric learning is the task of inferring a metric on a domain using side information, such as examplar points that should be close or far away from each other.
The assumption behind such methods is that metrics are chosen within parameterized families, and tailored for a task and data at hand, rather than selected among a few handpicked candidates.
Metric learning algorithms are supervised, often learning from similarity and dissimilarity constraints between pairs of samples ($x_i$ should be close to $x_j$), or triplets ($x_i$ is closer to $x_j$ than to $x_k$).
Metric learning has applications in different tasks, such as classification, image retrieval, or clustering.
For instance, for classification purposes, the learned metric brings closer samples of the same class and drives away samples of different classes~\cite{xing_distance_2003}.

Metric learning methods are either linear or non-linear, depending on the formulation of the metric with respect to its inputs.
We will briefly recall various metric learning approaches, but refer the reader to existing surveys~\cite{kulis_metric_2013,bellet_metric_2015}.
A widely-used linear metric function is the squared Mahalanobis distance, which is employed in the popular Large Margin Nearest Neighbors algorithm (LMNN) \cite{weinberger_distance_2006} along with a $k$-NN approach.
Other linear methods \cite{chechik_online_2009,pele_interpolated_2016} choose not to satisfy all distance axioms (unlike the Mahalanobis distance) for more flexibility and because they are not essential to agree with human perception of similarities~\cite{bellet_metric_2015}.
Non-linear methods include the prior embedding of the data (kernel trick) before performing a linear method~\cite{torresani_large_2007,wang_metric_2011}, or other non-linear metric functions~\cite{chopra_learning_2005,kedem_non-linear_2012}.
Facing problems where the data samples are histograms, researchers have developed metric learning methods based on distances that are better suited for histograms such as $\chi^2$ \cite{kedem_non-linear_2012,yang_chi-squared_2015} or the Wasserstein distance, which we describe in more detail.

\subsection{Ground Metric Learning}

The Wasserstein distance relies heavily---one could almost say exclusively---on the ground metric to define a geometry on probability distributions. Setting that parameter is therefore crucial, and being able to rely on an adaptive, data-based procedure to select it is attractive from an applied perspective. 
The ground metric learning (GML) problem, following the terminology set forth by \citet{cuturi_ground_2014}, considers the generic case in which a ground cost that is a true metric (definite, symmetric and satisfying triangle inequalities) is learned using supervised information from a set of histograms. This method requires projecting matrices onto the cone of metric matrices, which is known to require a cubic effort in the size of these matrices~\cite{brickell_metric_2008}.
\citet{wang_supervised_2012} follow GML's approach but drop the requirement that the learned cost must be a metric. %
\citet{zen_simultaneous_2014} use GML to enhance previous results on Non-negative Matrix Decomposition with a Wasserstein loss (EMD-NMF) \cite{sandler_nonnegative_2011}, by alternatively learning the matrix decomposition and the ground metric.
\new{
Learning a metric from the observation of a matching is a well-studied problem.
\citet{dupuy_estimating_2016} learn a similarity matrix from the observation of a fixed transport plan, and use this to propose factors explaining weddings across groups in populations.
\citet{stuart_inverse_2019} infer graph-based cost functions similar to ours, but learn from noisy observations of transport plans, in a Bayesian framework.
Finally, \citet{li_learning_2019} extend the method of \citet{dupuy_estimating_2016} for noisy and incomplete matchings, by relaxing the marginal constraint using the regularized Wasserstein distance itself (instead of the $\KL$ divergence commonly used in unbalanced optimal transport \cite{chizat_scaling_2016}).
}
\citet{huang_supervised_2016} consider a non-discrete GML problem that involves point-clouds, and propose to learn a Mahalanobis metric between word embeddings that agree with labels between texts, seen here as bags of words. Both of these approaches use the entropic regularization of Wasserstein distances (see next section).
More recently, \citet{xu_multi-level_2018} combined several previous ideas to create a new metric learning algorithm.
It is a regularized Wasserstein distance-flavored LMNN scheme, with
a Mahalanobis distance as ground metric, parameterized by multiple local metrics \cite{weinberger_fast_2008} and a global one.

Similarly to the above works, our method aims to learn the ground metric of OT distances, but differs in the formulation of the ground metric.
We search for metrics that are geodesic distances on graphs, via a diffusion equation.
Our method also differs in the data we learn from: the observations that are fed to our algorithm are snapshots of a mass movement, and not pair or triplet constraints.
We use displacement interpolations to reconstruct that movement, hence our objective function contains multiple inverse problems involving OT distances.
This contrasts with simpler formulations where the objective function or the constraints are weighted sums of OT distances \cite{cuturi_ground_2014,wang_supervised_2012,xu_multi-level_2018}.
Furthermore, the interest in these previous works was generally to perform supervised classification, which is something we do not aim to do in this paper.
Nevertheless, our learning algorithm is supervised, since we provide the exact timestamps of each sample in the sequence.

Our method is similar to the work of \citet{zen_simultaneous_2014}, in the sense that we both aim to reconstruct the input data given a model. However they only use OT distances as loss functions to compare inputs with linear reconstructions, while we use OT distances to synthesize the reconstructions themselves.
The difference between our method and those of \citet{dupuy_estimating_2016,stuart_inverse_2019,li_learning_2019} lies in the available observations: they learn from a fixed matching (transport plan) whereas we learn from a sequence of mass displacements from which we infer both the metric and an optimal transport plan.
Our method thus does not require identifying information on traveling masses.
The fact that this identification is not required ranks among the most important and beneficial contributions of the OT geometry to data sciences, notably biology~\cite{schiebinger_optimal-transport_2019}.

Our approach to metric learning corresponds to setting up an optimization problem over the space of geodesic distances on graphs, which is closely related to the continuous problem of optimizing Riemannian metrics.
Optimizing metrics from functionals involving geodesic distances has been considered in~\cite{benmansour_derivatives_2010}.
This has recently been improved in~\cite{mirebeau_automatic_2017} using automatic differentiation to compute the gradient of the functional involved, which is also the approach we take in our work.
This type of metric optimization problem has also been studied within the OT framework (see for instance~\cite{buttazzo_optimal_2004}), but these works are only concerned with convex problems (typically maximization of geodesic or OT distances), while our metric learning problem is highly non-convex.

\section{Context} %
\label{sec:context}

Optimal transport defines geometric distances between probability distributions.
In this paper, we consider discrete measures on graphs.
These measures are sums of weighted Dirac distributions supported on the graph's vertices: $\mu = \sum_{i=1}^{N} u_i \delta_{x_i}$, with the weight vector $\vec{u}=(u_i)$ in the probability simplex $\Sigma_N  \eqdef \left\lbrace \vec{u} \in \mathbb{R}^N_+ \big| \sum_{i=1}^{N}\vec{u}_i = 1\right\rbrace$, and $x_i$ the position of vertex $i$ in an abstract space.
In the following, we will refer to the weight vector of these measures as ``histograms''.

A transport plan between two histograms $\vec{a},\vec{b} \in \Sigma_N$ is a matrix $\mat{P} \in \mathbb{R}^{N\times N}_+$, where $\mat{P}_{i,j}$ gives the amount of mass to be transported from vertex $i$ of $\vec{a}$ to vertex $j$ of $\vec{b}$.
We define the transport polytope of $\vec{a}$ and $\vec{b}$ as
\begin{align*}
U(\vec{a},\vec{b}) \eqdef \Bigl\lbrace \mat{P} \in \mathbb{R}^{N\times N}_+ ~\mid~ \mat{P}\mathds{1}_N = \vec{a} ~~\textnormal {and}~~ \mat{P}^T\mathds{1}_N = \vec{b} \Bigr\rbrace.
\end{align*}

\paragraph{The Kantorovich problem.}

Optimal transport aims to find the transport plan $\mat{P}$ that minimizes a total cost, which is the mass transported multiplied by its cost of transportation.
This is called the Kantorovich problem and it is written
\begin{align}\label{eq:defkantoprob}
\sclr{W}_\mat{C}(\vec{a},\vec{b}) \eqdef \min_{\mat{P} \in U(\vec{a},\vec{b})} \sum_{i,j} \mat{C}_{i,j}\mat{P}_{i,j}.
\end{align}
The cost matrix $\mat{C} \in \mathbb{R}^{N\times N}_+$ defines the cost $\mat{C}_{i,j}$ of transporting one unit of mass from vertex $i$ to $j$.
If the cost matrix is $\mat{C}_{i,j} = d(x_i,x_j)^p$, with $d$ a distance on the domain, then $\sclr{W}_\mat{C}^{1/p}$ is a distance between probability distributions, called the $p$-Wasserstein distance \cite[Proposition 2.2]{peyre_computational_2018}.

\paragraph{Entropy regularization.}

This optimization problem can be regularized, and a computationally efficient way to do so is to balance the transportation cost with the entropy $H$ of the transport plan \cite{cuturi_sinkhorn_2013}.
The resulting entropy-regularized problem is written as
\begin{align}\label{eq:defregwassdist}
\sclr{W}^\varepsilon_\mat{C}(\vec{a},\vec{b}) \eqdef \min_{\mat{P} \in U(\vec{a},\vec{b})} \left< \mat{C},\mat{P}\right> - \varepsilon H(\mat{P}),
\end{align}
where $H(\mat{P}) \eqdef -\sum_{i,j} \mat{P}_{i,j}(\log(\mat{P}_{i,j}) -1)$ and $\varepsilon>0$.
The addition of this regularization term modifies how the Kantorovich problem can be addressed.
Without regularization, the problem must be solved with network flow solvers, whereas with regularization it can be conveniently solved using the Sinkhorn algorithm, which is less costly than minimum cost network flow algorithms for sufficiently large values of $\varepsilon$ (the smaller the regularization, the slower the convergence).
The obtained value $\sclr{W}^\varepsilon_\mat{C}$ is an approximation of the exact Wasserstein distance, and the approximation error can be controlled with $\varepsilon$.
Another chief advantage of this regularization is that $\sclr{W}^\varepsilon_\mat{C}(\vec{a},\vec{b})$ defines a smooth function of both its inputs $(\vec{a},\vec{b})$ and the metric $C$.
This property is important to be able to derive efficient and stable metric learning schemes, as we seek in this paper.

\paragraph{Displacement interpolation.}

Given two histograms $\vec{r}_0$ and $\vec{r}_1$, their barycentric interpolation is defined as the curve parameterized for $t \in [0,1]$ as
\begin{align}\label{eq:dispinterpprob}
	\vec{\gamma}_\mat{C}(\vec{r}_0,\vec{r}_1,t) \eqdef \argmin_{\vec{r}\in\Sigma_N} ~ (1-t)W_\mat{C}(\vec{r}_0,\vec{r}) + tW_\mat{C}(\vec{r},\vec{r}_1).
\end{align}
This class of problem was introduced and studied by~\cite{agueh_barycenters_2011}.
In the case where $d$ is a geodesic distance and $p=2$, then $W_\mat{C}(\vec{r}_0,\vec{r}_1)$ not only represents the total cost of transporting $\vec{r}_0$ to $\vec{r}_1$, but also the square of the length of the shortest path (a geodesic) between them in the Wasserstein space.
In this case,~\eqref{eq:dispinterpprob} defines the so-called \emph{displacement interpolation}~\cite{mccann_convexity_1997}, which is also the geodesic from $\vec{r}_0$ to $\vec{r}_1$.

With a slight abuse of notation, in the following we call $\vec{\gamma}_\mat{C}$ the displacement interpolation path, for any generic cost $C$.
In practice, we approximate this interpolation using the regularized Wasserstein distance $\sclr{W}^\varepsilon_\mat{C}$, which means we can compute it as a special case of regularized Wasserstein barycenter \cite{benamou_iterative_2015} between two histograms.
We denote $\vec{\gamma}_\mat{C}^\varepsilon$ the resulting smoothed approximation.

\section{Method}
\label{sec:method}

\subsection{Metric parametrization}\label{subsec:metric_param}

Since histograms are supported on a graph, we parameterize the ground metric by a positive weight $w_{i,j}$ associated to each edge connecting vertices $i$ and $j$.
This should be understood as being inversely proportional to the length of the edge, and conveys how easily mass can travel through it.
Additionally, we set $w_{i,j}=0$ when vertices $i$ and $j$ are not connected.

We aim to carry out metric learning using OT where the ground cost is the square of the geodesic distance associated to the weighted graph.
Instead of optimizing a full adjacency matrix $\mat{W} = (w_{i,j})_{i,j}$, which has many zero entries that we do not wish to optimize, we define the vector $\vec{w} \in \mathbb{R}^\sclr{K}$ as the concatenation of all metric parameters $w_{i,j} > 0$, that is, those for which vertices $i$ and $j$ are connected.
This imposes a fixed connectivity on the graph.

\subsection{Problem statement}\label{subsec:prob_statement}

Let $(\vec{h}_i)_{i=1}^\sclr{P} \in \Sigma_N$ be observations at $\sclr{P}$ consecutive time steps of a movement of mass.
We aim to retrieve the metric weights $\vec{w}$ for which an OT displacement interpolation approximates best this mass evolution.
This corresponds to an OT regression scheme parameterized by the metric, and leads to the following optimization problem
\begin{align}\label{eq:mainpbmdisc}
\min_{\vec{w}}~\sum_{i=1}^{\sclr{P}} \mathcal{L}\left(\vec{\gamma}_{{\mat{C}}_{\vec{w}}}^\varepsilon \left( \vec{h}_1,\vec{h}_{\sclr{P}},t_i\right), \vec{h}_i\right) + f(\vec{w}),
\end{align}
where:
\begin{itemize}
	\item $t_i$ are the timestamps: $t_i = (i-1)/(\sclr{P}-1)$,
	\item $\mathcal{L}$ is a loss function between histograms,
	\item $\mat{C}_{\vec{w}}$ is the ground metric matrix associated to the graph weights $\vec{w}$, detailed in section~\ref{subsec:computegeodist},
	\item and $f(\vec{w})$ is a regularization term detailed in section~\ref{subsec:regul}.
\end{itemize}
Note that we chose equally spaced timestamps for the sake of simplicity, but the method is applicable to any sequence of timestamps.

In the following sections, we detail the different components of our algorithm.
Our objective function \eqref{eq:mainpbmdisc} is non-convex, and we minimize it with an L-BFGS quasi-Newton algorithm, to compute a local minimum of the non-convex energy.
The L-BFGS algorithm requires the evaluation of the energy function, as well as its gradient with respect to the inputs.
In our case, evaluating the energy function \eqref{eq:mainpbmdisc} requires reconstructing the sequence of input histograms using a displacement interpolation between the first and last histogram, and assessing the quality of the reconstructions.
The gradient is calculated through automatic differentiation, which provides high flexibility when adjusting the framework.

In our numerical examples, we consider 2-D and \nobreakdash{3-D} datasets discretized on uniform square grids, so that the graph is simply the graph of 4 or 6 nearest neighbors on this grid.

\subsection{Kernel application}\label{subsec:recinputs}

As mentioned previously, we use entropy-regularized OT to compute displacement interpolations \eqref{eq:dispinterpprob}.
These interpolations can be computed via the Sinkhorn barycenter algorithm \ref{alg:sinkbary}, and for which the main computational burden is to apply the kernel matrix $\mat{K}$ on $\mathbb{R}^{N}$ vectors. %
\new{
When the domain is a grid and the metric is Euclidean, applying that kernel boils down to a simple convolution with a Gaussian kernel. %
For an arbitrary metric as in our case, computing the kernel $\mat{K}$ requires all-pairs geodesic distances on the graph.
This can be achieved using e.g. Dijkstra's algorithm or the Floyd-Warshall algorithm.
However, this kernel is a non-smooth operator, which is quite difficult to differentiate with respect to the metric weights (see for instance~\cite{benmansour_derivatives_2010,mirebeau_automatic_2017} for works in this direction).
In sharp contrast, our approach leverages Varadhan's formula~\cite{varadhan_behavior_1967}: we approximate the geodesic kernel $\mat{K}$ with the heat kernel.
This kernel is itself approximated by solving the diffusion equation using $\sclr{S}$ sub-steps of an implicit Euler scheme.
This has the consequence of 1) having faster evaluations of the kernel and 2) smoothing the dependency between the distance kernel and the metric.
This last property is necessary to define a differentiable functional, which is key for an efficient solver.
This has been applied to OT computation by~\citet{solomon_convolutional_2015}, inspired from the work of~\citet{crane_geodesics_2013}.
}

\makeatletter
\renewcommand{\ALG@beginalgorithmic}{\large}
\makeatother
\begin{algorithm}[H]
	\caption{Sinkhorn barycenter \cite{benamou_iterative_2015}}
	\label{alg:sinkbary}
	\begin{flushleft}
		\normalsize
		\textbf{Input:} histograms $(\vec{a}_r) \in (\Sigma_N)^R$, weights $\vec{\lambda} \in \Sigma_R$, kernel $\mat{K}$, number of iterations $\sclr{L}$ \\
		\textbf{Ouput:} barycenter $\vec{b}$
	\end{flushleft}
	\begin{algorithmic}
		\State $\forall r, \vec{v}_r = \vec{1}_{N}$
		\For {$l = 1$ to $\sclr{L}$}
			\State $\forall r, \vec{u}_r = \frac{\vec{a}_r}{\mat{K}\vec{v}_r}$
			\State $\vec{b} = \prod_r \left(\mat{K}^T\vec{u}_r\right)^{\vec{\lambda}_r}$
			\State $\forall r, \vec{v}_r = \frac{\vec{b}}{\mat{K}^T\vec{u}_r}$
		\EndFor
	\end{algorithmic}
\end{algorithm}

\subsection{Computing geodesic distances}\label{subsec:computegeodist}

While \citet{solomon_convolutional_2015} discretize the diffusion equation using a cotangent Laplacian because they deal with triangular meshes, we prefer a weighted graph Laplacian parameterized by the metric weights $\vec{w}$, which we detail hereafter.

As mentioned in section \ref{subsec:metric_param}, the weighted adjacency matrix $\mat{W}$ is defined as $\mat{W}_{i,j}=\mat{W}_{j,i}=w_{i,j}$ where $w_{i,j}$ are the (undirected) edge weights parameterizing the metric.
It is symmetric and usually sparse, since $w_{i,j}$ is non-zero only for vertices that are connected, and $0$ otherwise.
The diagonal weighted degree matrix sums the weights of each row on the diagonal: $\mat{\Lambda} \eqdef \diag(\vec{d})$, with  $\vec{d}_{i} \eqdef \sum_{j=1}^{N} w_{i,j}$.
The negative semi-definite weighted graph Laplacian matrix is then defined as $\mat{L}_{\vec{w}} = \mat{W} - \mat{\Lambda}$.

We discretize the heat equation in time using an implicit Euler scheme and perform $\sclr{S}$ sub-steps.
It is crucial to rely on an implicit stepping scheme to obtain approximated kernels supported on the full domain, in order for Sinkhorn iterations to be well conditioned (as opposed to using an explicit Euler scheme, which would break Sinkhorn's convergence).
Denoting $\vec{v}$ the initial condition of the heat diffusion, $\vec{u}$ the final solution after a time $\varepsilon/4$, and $\mat{L}_{\vec{w}}$ our discrete Laplacian operator, we solve 
\begin{align}\label{eq:applykernelMk}
\left(\Id - \frac{\sclr{\varepsilon}}{4\sclr{S}}\mat{L}_{\vec{w}} \right)^\sclr{S}\vec{u} &= \vec{v}.
\end{align}

We denote by $\mat{M}$ the symmetric matrix $\Id - \frac{\varepsilon}{4\sclr{S}}\mat{L}_{\vec{w}}$.
Applying the kernel $\mat{K} \eqdef \mat{M}^{-\sclr{S}}$ to a vector $\vec{v}$ is then simply achieved by solving $\sclr{S}$ linear systems: $\vec{u} = \mat{K}\vec{v} = \mat{M}^{-\sclr{S}}\vec{v}$.
We never compute the full kernel matrix $\mat{K}$ because it is of size $N^2$, which quickly becomes prohibitive in time and memory as histograms grow ($\approx$ 12GB for histograms of size $N=200^2$, and $\approx$ 30GB for histograms of size $N=40^3$).

The intuition behind this scheme is that, $(\Id - \frac{\varepsilon}{4\sclr{S}}\mat{L}_{\vec{w}})^{-\sclr{S}}$ approximates the heat kernel for large $\sclr{S}$, which itself for small $\varepsilon$ approximates the geodesic exponential kernel, which is of the form $\exp(-d^2(x,y)/\varepsilon)$ for a small $\varepsilon$ with $d$ the geodesic distance on a manifold. %
Note however that this link is not valid on graphs or triangulations, although it has been reported to be very effective when choosing $\varepsilon$ in proportion to the discretization grid size (see~\cite{solomon_convolutional_2015}).
Our method can thus be seen as choosing a cost of the form \begin{align}\label{eq:defcostheat}
\mat{C}_{\vec{w}} \eqdef -\varepsilon \log\left( (\Id - \frac{\varepsilon}{4\sclr{S}} \mat{L}_{\vec{w}} )^{-\sclr{S}}\right),
\end{align}
even though we never compute this cost matrix explicitly.

The chief advantage of the formula~\eqref{eq:applykernelMk} to approximate a kernel evaluation is that the same matrix is repeatedly used $\sclr{S}$ times, which is itself repeated at each iteration of Sinkhorn's algorithm~\ref{alg:sinkbary} to evaluate barycenters.
Following~\citet{solomon_convolutional_2015}, a dramatic speed-up is thus obtained by pre-computing a sparse Cholesky decomposition of $\mat{M}$.
For instance, on a 2-D domain, the number of non-zero elements of such a factorization is of the order of $N$, so that each linear system resolution has linear complexity.

\subsection{Inverse problem regularization}
\label{subsec:regul}

The metric learning problem is severely ill-posed and this difficulty is further increased by the fact that the corresponding optimization problem~\eqref{eq:mainpbmdisc} is non-convex.
These issues can be mitigated by introducing a regularization term $f(\vec{w}) \eqdef \lambda_c f_c(\vec{w}) + \lambda_s f_s(\vec{w})$.
Note also that since a global variation of scale in the metric does not affect the solution of optimal transport, the problem needs to be constrained, otherwise metric weights tend to infinity when they are optimized.

We introduce two different regularizations: $f_c$ forces the weights to be close to 1 (this controls how much the space becomes inhomogeneous and  anisotropic), and $f_s$ constrains the weights to be spatially smooth.
Since we carry out the numerical examples on graphs that are 2-D and 3-D grids, we use a smoothing regularization $f_s$ that is specific to that case.
This term must be adapted when dealing with general graphs.

The first regularization is imposed by adding the following term to our energy functional, multiplied by a control coefficient~$\lambda_c$:
\begin{align}
f_c(\vec{w}) \eqdef \norm{\vec{w}-\mathds{1}}_2^2.
\end{align}

To enforce the second prior, we add the following term to our functional, multiplied by a control coefficient~$\lambda_s$:
\begin{align}
	f_s(\vec{w}) \eqdef \sum_{e \in E} ~\left(\sum_{e'\in\mathcal{N}_\parallel(e)}(w_e - w_{e'})\right)^2,
\end{align}
with $E$ the set of undirected edges, and $\mathcal{N}_\parallel$ the set of neighbor edges of the same orientation, as illustrated in \autoref{fig:regul_edge} for the 2-D case.

We regularize separately horizontal and vertical edges to ensure that we recover an anisotropic metric.
This is important for various applications, for example when dealing with color histograms, as MacAdam's ellipses reveal \cite{macadam_visual_1942}.

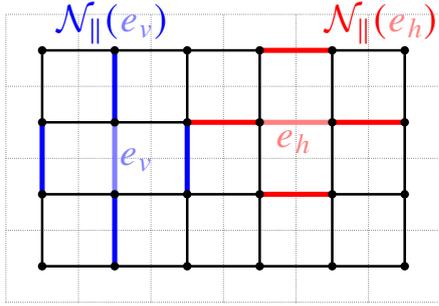
\begin{figure}
	\centering
	\begin{overpic}[width=0.7\linewidth]{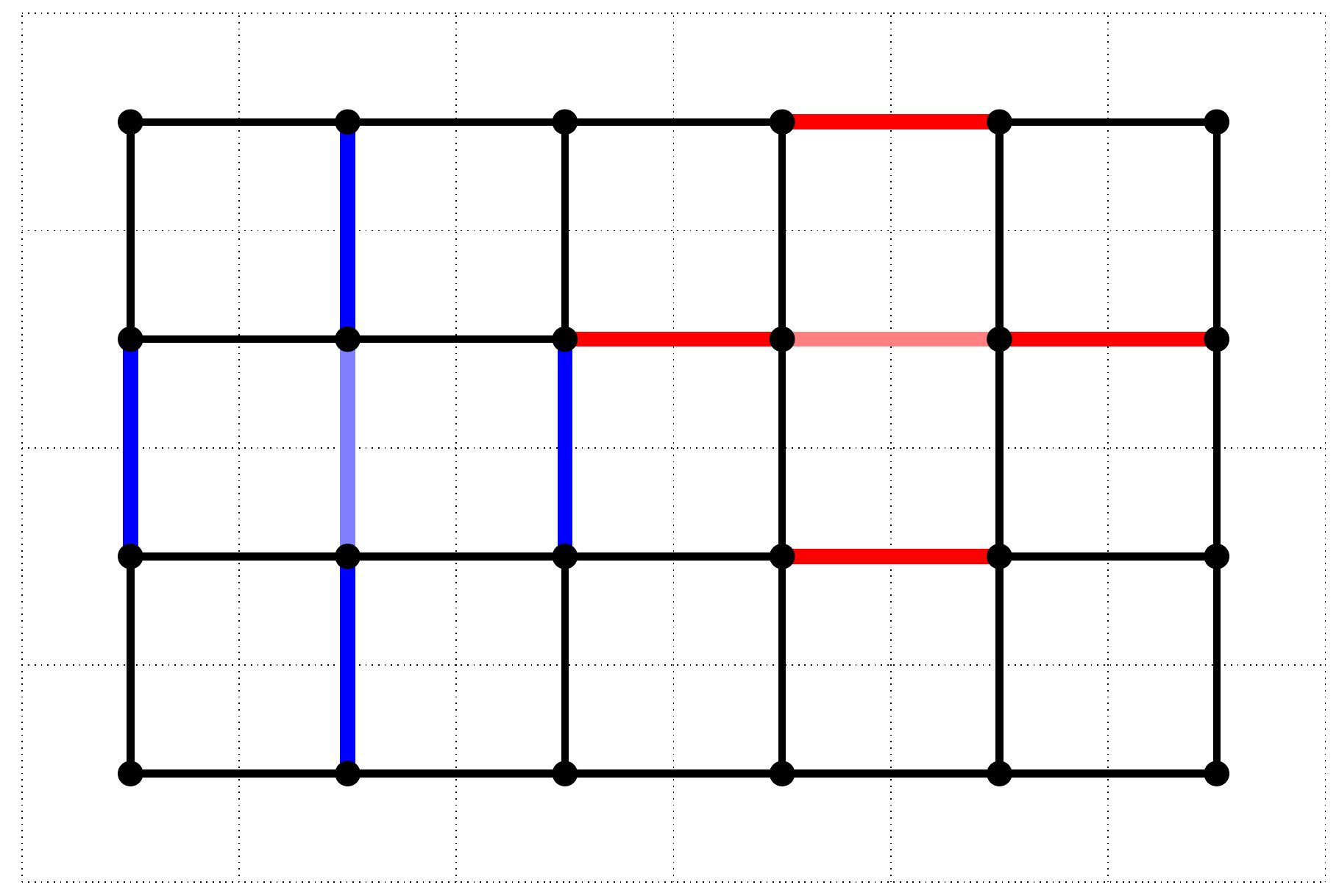}%
		\let\oldput\put
		\def\put(#1,#2)#3{%
			\oldput(#1,#2){\Large #3}%
		}
		\put(27,32){\color{blue!50!white}$e_v$}
		\put(62,36){\color{red!50!white}$e_h$}
		\put(13,62){\color{blue} $\mathcal{N}_\parallel({\color{blue!50!white}e_v})$}
		\put(73,62){\color{red} $\mathcal{N}_\parallel({\color{red!50!white}e_h})$}
	\end{overpic}
	\caption{Smooth prior: for each vertical ($e_v$) and horizontal ($e_h$) edge, we minimize the squared sum of weight differences with its respective neighbors of the same orientation.} %
	\label{fig:regul_edge}
\end{figure}

The selection of the regularization parameters $(\lambda_c,\lambda_s)$ and their impact on the recovered metric is discussed in section \ref{sec:discussion}.

\subsection{Implementation}
\label{subsec:implementation}

In order to ensure positivity of the metric weights, problem~\eqref{eq:mainpbmdisc} is solved after a log-domain change of variable $\vec{w} = e^{\vec{w'}}$ and the optimization on $\vec{w'}$ is achieved using the L-BFGS algorithm.

Our method is implemented in Python with the Pytorch framework, which supports automatic differentiation (AD)~\cite{paszke_automatic_2017}.
The gradient is evaluated using reverse mode automatic differentiation~\cite{griewank_who_2012,griewank_evaluating_2008}, which has a numerical complexity of the same order as that of evaluating the minimized functional (which corresponds to the evaluation of $\sclr{P}$ barycenters).
\new{
Reverse mode automatic differentiation computes the gradient of the energy~\eqref{eq:mainpbmdisc} by back-propagating the gradient of the loss function $\mathcal{L}$ through the computational graph defined by the Sinkhorn barycenter algorithm.
This back-propagation operates by applying the adjoint Jacobian of each operation involved in the algorithm.
These adjoints Jacobians are applied in a backward pass, in an ordering that is the reverse of the one used to perform the forward pass of the algorithm. 
Most of these operations are elementary functions, such as matrix product and pointwise operations (e.g. division or exponentiation).
These functions are built-in for most standard automatic differentiation libraries, so their adjoint Jacobians are already implemented.
The only non-trivial case, which is usually not built-in, and can also be optimized for memory usage, is the application of the heat kernel $\mat{M}^{-\sclr{S}}$ to some vector $\vec{v}$.

The following proposition gives the expression of the adjoint Jacobian $[\partial \Phi(\vec{w})]^T : \mathbb{R}^N \rightarrow \mathbb{R}^\sclr{K}$ of the map $\Phi : \mathbb{R}^\sclr{K} \rightarrow \mathbb{R}^N$ 
\begin{equation}\label{eq-func-weights-to-heat}
\Phi : \vec{w} \in \mathbb{R}^{\sclr{K}} \mapsto ( \mat{M} )^{-\sclr{S}} \vec{v} \in \mathbb{R}^N
\;\text{where}\;
\mat{M} \eqdef {\Id} - \frac{\sclr{\varepsilon}}{4\sclr{S}} \mat{L}_{\vec{w}},
\end{equation}
from the graph weights $\vec{w}$ to the approximate solution of the heat diffusion, obtained by applying $\sclr{S}$ implicit Euler steps starting from $\vec{v} \in \mathbb{R}^N$.

\begin{proposition}
	For $\vec{w} \in \mathbb{R}^{\sclr{K}}$ and $\vec{g} \in \mathbb{R}^N$, one has for each of the $\sclr{K}$ edge indices $(i,j)$ 
	\begin{align}\label{eq-adjoint-formula}
	&[\partial \Phi(\vec{w})]^T(\vec{g})_{i,j}
	= 
	\frac{\varepsilon}{4\sclr{S}}\sum_{\ell=0}^{\sclr{S}-1} \left(\vec{g}_i^\ell - \vec{g}_j^\ell \right) \left(\vec{v}_i^\ell - \vec{v}_j^\ell \right) \\
	&\text{where}\quad
	\left\{
	\begin{array}{l}
	\vec{g}^\ell \eqdef \mat{M}^{\ell-\sclr{S}}\vec{g} \\
	\vec{v}^\ell \eqdef \mat{M}^{-\ell-1} \vec{v}
	\end{array}
	\right.
	\end{align}
\end{proposition}

\begin{proof}
	The mapping $\Phi$ is the composition of $\phi_1: \vec{w} \mapsto \mat{M}$, $\phi_2: \mat{M} \mapsto \mat{U} = \mat{M}^{-1}$, $\phi_3: \mat{U} \mapsto \mat{V} = \mat{U}^{\sclr{S}}$ and $\phi_4: \mat{V} \mapsto \vec{y} = \mat{V}\vec{v}$. The formula follows by composing the adjoint Jacobians of each of these operations, which are detailed in \autoref{apx:kernelgrad}. \qed
\end{proof}

Note that the vectors $\vec{v}_\ell$ are already computed during the forward pass of the algorithm, so they need to be stored. Only the vectors $\vec{g}_\ell$ need to be computed by iteratively solving linear systems associated to $\mat{M}$.

It is important to set a fixed number of iterations $L$ for the Sinkhorn barycenter algorithm when using automatic differentiation, because a stopping criteria based on convergence might be problematic to differentiate (it is a discontinuous function). Moreover, it would render memory consumption (and speed) unpredictable, which may be problematic due to the limited amount of memory available.

}

\section{Experiments}
\label{sec:experiments}

We first show a few synthetic examples, in which the input sequence of measures has been generated as a Wasserstein geodesic using a ground metric known beforehand.
This ground truth metric is compared with the output of our algorithm. 
We then present an application to a task of learning color variations in image sequences.

In the following, an ``interpolation'' refers to a displacement interpolation, unless stated otherwise.

\subsection{Synthetic experiments}

As mentioned in \ref{subsec:prob_statement}, our algorithm solves an inverse problem: given a sequence of histograms representing a movement of mass, we aim at fitting a metric for which that sequence can be sufficiently well approached by a displacement interpolation between the first and last frame.

\subsubsection{Retrieving a ground truth metric}
In \autoref{fig:toyinverseprob1}, \autoref{fig:toyinverseprob2} and \autoref{fig:toyinverseprob3}, we test our algorithm by applying it on different sequences of measures that are themselves geodesics generated using handcrafted metrics, and verify that the learned metric is close to the original one, which constitutes a ground truth.
In general, it is impossible to recover with high precision the exact same metric, because such an inverse problem is too ill-posed (many different metrics can generate the same interpolation sequence) and the energy is non-convex.
Moreover, regularization introduces a bias while helping to fight against this non-convexity.
Hence, we attempt to find a metric that shares the same large scale features as the original one.

\begin{figure*}
	\centering
	\newcommand{\imhspace}{\dimexpr 2pt}
	\newcommand{\imvspace}{\dimexpr -1pt + 2pt}
	\newcommand{\imheighta}{\dimexpr(0.97\linewidth-\imhspace*10)*265/3308}
	\newcommand{\imheightb}{\dimexpr(0.97\linewidth-\imhspace*10)*265/3308}

	Horizontal \hspace{1.5em} Vertical \hfill \vspace{0.5em} \\
	\adjincludegraphics[height=\imheighta,trim={{0.11\width} {0.06\height} {0.05\width} {0.12\height}}, clip]{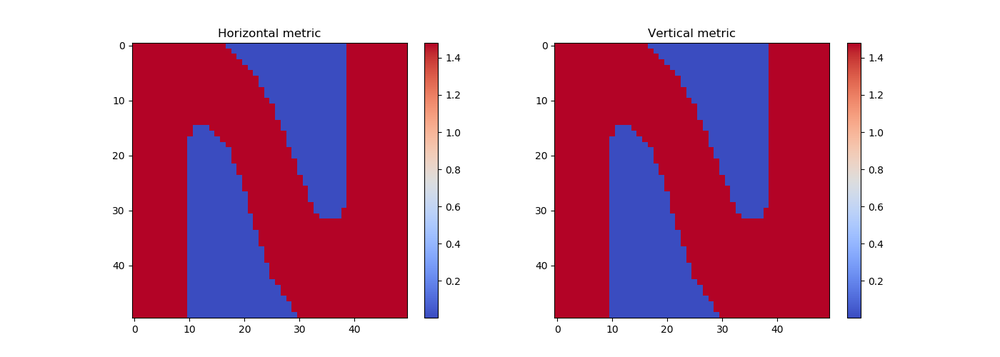}\hspace{\imhspace}
	\adjincludegraphics[height=\imheighta]{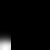}\hspace{\imhspace}
	\adjincludegraphics[height=\imheighta]{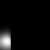}\hspace{\imhspace}
	\adjincludegraphics[height=\imheighta]{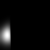}\hspace{\imhspace}
	\adjincludegraphics[height=\imheighta]{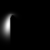}\hspace{\imhspace}
	\adjincludegraphics[height=\imheighta]{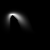}\hspace{\imhspace}
	\adjincludegraphics[height=\imheighta]{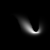}\hspace{\imhspace}
	\adjincludegraphics[height=\imheighta]{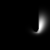}\hspace{\imhspace}
	\adjincludegraphics[height=\imheighta]{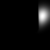}\hspace{\imhspace}
	\adjincludegraphics[height=\imheighta]{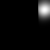}\hspace{\imhspace}
	\adjincludegraphics[height=\imheighta]{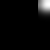} \\
	\vspace{\imvspace}
	\adjincludegraphics[height=\imheightb,trim={{0.11\width} {0.06\height} {0.05\width} {0.12\height}}, clip]{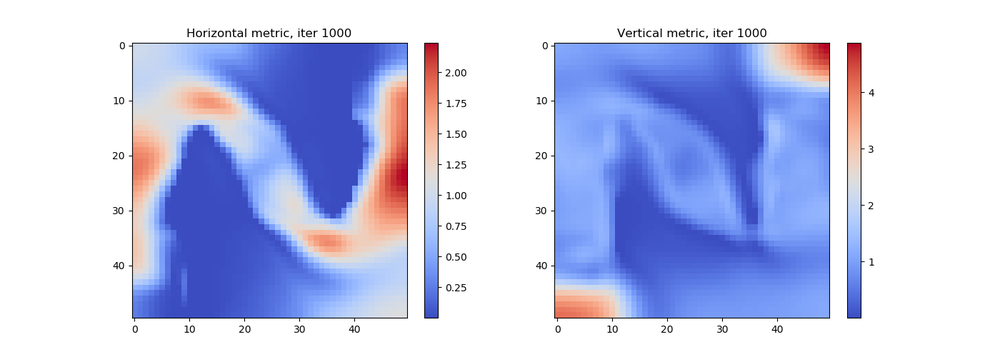}\hspace{\imhspace}
	\adjincludegraphics[height=\imheightb]{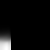}\hspace{\imhspace}
	\adjincludegraphics[height=\imheightb]{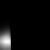}\hspace{\imhspace}
	\adjincludegraphics[height=\imheightb]{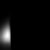}\hspace{\imhspace}
	\adjincludegraphics[height=\imheightb]{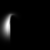}\hspace{\imhspace}
	\adjincludegraphics[height=\imheightb]{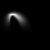}\hspace{\imhspace}
	\adjincludegraphics[height=\imheightb]{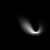}\hspace{\imhspace}
	\adjincludegraphics[height=\imheightb]{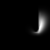}\hspace{\imhspace}
	\adjincludegraphics[height=\imheightb]{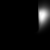}\hspace{\imhspace}
	\adjincludegraphics[height=\imheightb]{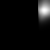}\hspace{\imhspace}
	\adjincludegraphics[height=\imheightb]{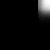} \\

	\caption{First row: an initial metric (two leftmost images: weights on horizontal and vertical edges) is used to generate a histogram sequence.
	Second row: we apply our algorithm on that sequence, to recover the initial metric.
	The algorithm is able to recover the blue zones avoided by the mass, and red zones on the path it is taking.}
	\label{fig:toyinverseprob1}
\end{figure*}

\begin{figure*}
	\centering
	\newcommand{\imhspace}{\dimexpr 2pt}
	\newcommand{\imvspace}{\dimexpr -1pt + 2pt}
	\newcommand{\imheighta}{\dimexpr(0.97\linewidth-\imhspace*10)*265/3308}
	\newcommand{\imheightb}{\dimexpr(0.97\linewidth-\imhspace*10)*265/3308}

	Horizontal \hspace{1.5em} Vertical \hfill \vspace{0.5em} \\
	\adjincludegraphics[height=\imheighta,trim={{0.11\width} {0.06\height} {0.05\width} {0.12\height}}, clip]{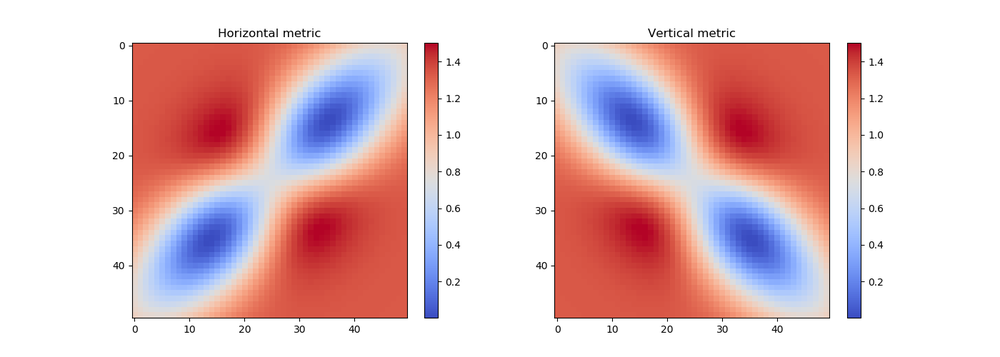}\hspace{\imhspace}
	\adjincludegraphics[height=\imheighta]{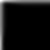}\hspace{\imhspace}
	\adjincludegraphics[height=\imheighta]{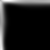}\hspace{\imhspace}
	\adjincludegraphics[height=\imheighta]{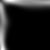}\hspace{\imhspace}
	\adjincludegraphics[height=\imheighta]{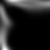}\hspace{\imhspace}
	\adjincludegraphics[height=\imheighta]{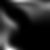}\hspace{\imhspace}
	\adjincludegraphics[height=\imheighta]{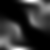}\hspace{\imhspace}
	\adjincludegraphics[height=\imheighta]{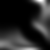}\hspace{\imhspace}
	\adjincludegraphics[height=\imheighta]{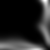}\hspace{\imhspace}
	\adjincludegraphics[height=\imheighta]{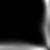}\hspace{\imhspace}
	\adjincludegraphics[height=\imheighta]{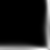} \\
	\vspace{\imvspace}
	\adjincludegraphics[height=\imheightb,trim={{0.11\width} {0.06\height} {0.05\width} {0.12\height}}, clip]{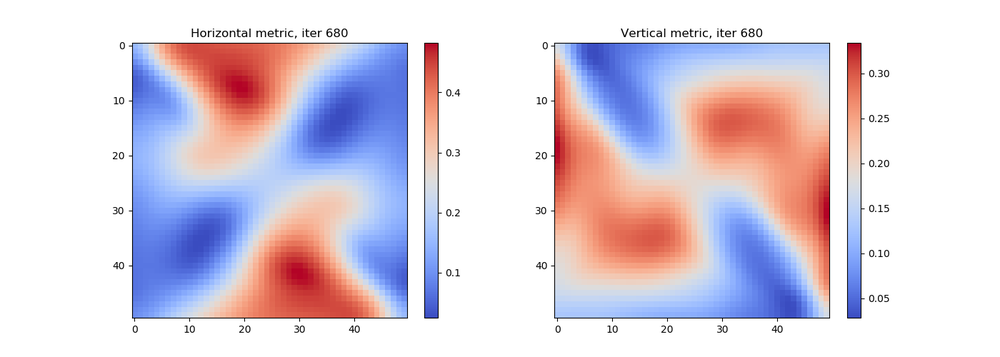}\hspace{\imhspace}
	\adjincludegraphics[height=\imheightb]{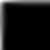}\hspace{\imhspace}
	\adjincludegraphics[height=\imheightb]{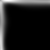}\hspace{\imhspace}
	\adjincludegraphics[height=\imheightb]{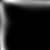}\hspace{\imhspace}
	\adjincludegraphics[height=\imheightb]{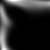}\hspace{\imhspace}
	\adjincludegraphics[height=\imheightb]{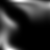}\hspace{\imhspace}
	\adjincludegraphics[height=\imheightb]{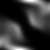}\hspace{\imhspace}
	\adjincludegraphics[height=\imheightb]{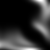}\hspace{\imhspace}
	\adjincludegraphics[height=\imheightb]{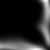}\hspace{\imhspace}
	\adjincludegraphics[height=\imheightb]{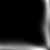}\hspace{\imhspace}
	\adjincludegraphics[height=\imheightb]{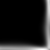} \\

	\caption{First row: an initial metric (two leftmost images: weights on horizontal and vertical edges) is used to generate a histogram sequence.
	Second row: we apply our algorithm on that sequence, to recover the initial metric.
	The algorithm recovers the high (red) and low (blue) diffusion areas horizontally, as well as vertically.}
	\label{fig:toyinverseprob2}
\end{figure*}

\begin{figure*}
	\centering
	\newcommand{\imhspace}{\dimexpr 2pt}
	\newcommand{\imvspace}{\dimexpr -1pt + 2pt}
	\newcommand{\imheighta}{\dimexpr(0.97\linewidth-\imhspace*10)*265/3308}
	\newcommand{\imheightb}{\dimexpr(0.97\linewidth-\imhspace*10)*265/3308}

	Horizontal \hspace{1.5em} Vertical \hfill \vspace{0.5em} \\
	\adjincludegraphics[height=\imheighta,trim={{0.11\width} {0.06\height} {0.05\width} {0.12\height}}, clip]{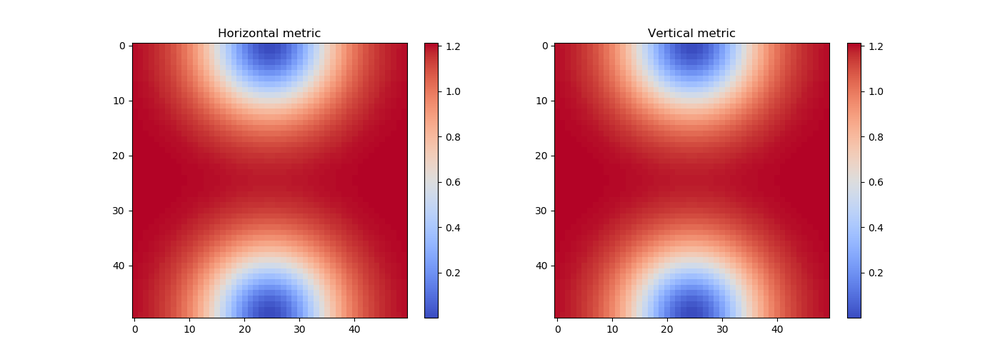}\hspace{\imhspace}
	\adjincludegraphics[height=\imheighta]{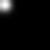}\hspace{\imhspace}
	\adjincludegraphics[height=\imheighta]{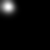}\hspace{\imhspace}
	\adjincludegraphics[height=\imheighta]{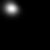}\hspace{\imhspace}
	\adjincludegraphics[height=\imheighta]{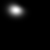}\hspace{\imhspace}
	\adjincludegraphics[height=\imheighta]{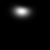}\hspace{\imhspace}
	\adjincludegraphics[height=\imheighta]{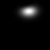}\hspace{\imhspace}
	\adjincludegraphics[height=\imheighta]{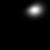}\hspace{\imhspace}
	\adjincludegraphics[height=\imheighta]{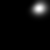}\hspace{\imhspace}
	\adjincludegraphics[height=\imheighta]{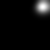}\hspace{\imhspace}
	\adjincludegraphics[height=\imheighta]{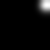} \\
	\vspace{\imvspace}
	\adjincludegraphics[height=\imheightb,trim={{0.11\width} {0.06\height} {0.05\width} {0.12\height}}, clip]{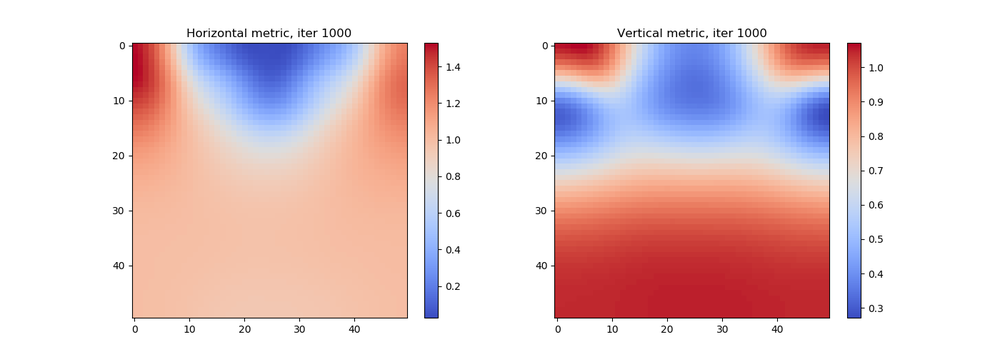}\hspace{\imhspace}
	\adjincludegraphics[height=\imheightb]{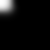}\hspace{\imhspace}
	\adjincludegraphics[height=\imheightb]{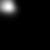}\hspace{\imhspace}
	\adjincludegraphics[height=\imheightb]{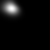}\hspace{\imhspace}
	\adjincludegraphics[height=\imheightb]{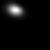}\hspace{\imhspace}
	\adjincludegraphics[height=\imheightb]{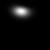}\hspace{\imhspace}
	\adjincludegraphics[height=\imheightb]{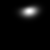}\hspace{\imhspace}
	\adjincludegraphics[height=\imheightb]{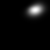}\hspace{\imhspace}
	\adjincludegraphics[height=\imheightb]{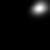}\hspace{\imhspace}
	\adjincludegraphics[height=\imheightb]{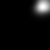}\hspace{\imhspace}
	\adjincludegraphics[height=\imheightb]{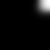} \\

	\caption{First row: an initial metric (two leftmost images: weights on horizontal and vertical edges) is used to generate a histogram sequence.
	Second row: we apply our algorithm on that sequence, to recover the initial metric.
	This figure shows an example of a metric detail not being recovered because mass is not traveling in that region.}
	\label{fig:toyinverseprob3}
\end{figure*}

We run three experiments with different handcrafted metrics, on 2-D histograms defined on an $n$ by $n$ Cartesian grid.
The parameters for these experiments are: a grid of size $n=50$, $\sclr{L}=50$ Sinkhorn iterations, an entropic regularization factor $\varepsilon$=1.2e-2, $\sclr{S}=100$ sub-steps for the diffusion equation, 1000 L-BFGS iterations, and the metric regularization factor $\lambda_c=0$.
The other regularization factor is $\lambda_s=0.03$ for the first two experiments and $\lambda_s=1.0$ for the third one.
Finally, each of the three experiments is tested with three different loss functions, and we display the result that is closest to the ground truth.
The different loss functions are the $L^1$ norm, the squared $L^2$ norm and the Kullback-Leibler divergence:
\begin{align}
\mathcal{L}_{1}(\vec{p},\vec{q}) &\eqdef \norm{\vec{p}-\vec{q}}_1, \label{eq:lossL1} \\
\mathcal{L}_{2}(\vec{p},\vec{q}) &\eqdef \norm{\vec{p}-\vec{q}}_2^2, \label{eq:lossL2} \\
\mathcal{L}_{KL}(\vec{p},\vec{q}) &\eqdef \mathds{1}^T (\vec{p}\odot\log(\vec{p}\oslash\vec{q})-\vec{p}+\vec{q}), \label{eq:lossKL} 
\end{align}
with $\odot$ and $\oslash$ being respectively the element-wise multiplication and division.
We will see that the best loss function varies depending on the data.

The metric $\vec{w}$ is located along either vertical or horizontal edges.
We thus display two images each time: one for the horizontal and one for the vertical edges.

In \autoref{fig:toyinverseprob1}, we are able to reconstruct the input sequence, and retrieve the main zones of low diffusion (in blue), that deviate the mass from a straight trajectory.
The $\mathcal{L}_1$ loss gave the best result.

In \autoref{fig:toyinverseprob2}, the original horizontal and vertical metric weights are different and this experiment shows that we are able to recover the distinct features of each metric \ie the dark blue and dark red areas.
The $\mathcal{L}_{KL}$ loss gave the best result.

In \autoref{fig:toyinverseprob3}, the original metric is composed of two obstacles, but only one of them is in the mass' trajectory.
We can observe that obstacles that are not approached by any mass are not recovered, which is expected, because the algorithm cannot find information in these areas.
The $\mathcal{L}_2$ loss gave the best result.

\new{
\subsubsection{Hand-crafted interpolations}

We now test our algorithm on a dataset that has not been generated by the forward model, as in the previous paragraph.
We reproduced in \autoref{fig:toydataMC} the trajectory of \autoref{fig:toyinverseprob3} using a moving Gaussian distribution.
By comparing them, we can see that reconstructions are close to the input, but not identical, since the metric induces an inhomogeneous diffusion, resulting in an apparent motion blur.
The horizontal metric obtained on this dataset is similar to the one obtained in \autoref{fig:toyinverseprob3}, with a blue obstacle on the upper middle part that prevents mass from going straight.
However, it differs in that the blue region extends towards the center of the support.
This can be explained by the fact that the algorithm tries to fit the input densities, which are less diffuse horizontally. Therefore, it will decrease the metric so that diffusion is less strong in that zone.
Parameters for this experiment were $n=50$, $\sclr{L}=50$, $\varepsilon$=1.2e-2, $\sclr{S}=100$, $\lambda_c=0$, $\lambda_s=10$, and an $\mathcal{L}_2$ loss.

}

\begin{figure*}
	\centering
	\newcommand{\imhspace}{\dimexpr 2pt}
	\newcommand{\imvspace}{\dimexpr -1pt + 2pt}
	\newcommand{\imheighta}{\dimexpr(0.97\linewidth-\imhspace*10)*449/5585}
	\newcommand{\imheightb}{\dimexpr(0.97\linewidth-\imhspace*10)*449/5585}
	
	\hspace{2.55\imheighta}\hspace{3\imhspace}
	\adjincludegraphics[height=\imheighta]{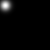}\hspace{\imhspace}
	\adjincludegraphics[height=\imheighta]{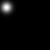}\hspace{\imhspace}
	\adjincludegraphics[height=\imheighta]{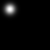}\hspace{\imhspace}
	\adjincludegraphics[height=\imheighta]{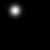}\hspace{\imhspace}
	\adjincludegraphics[height=\imheighta]{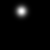}\hspace{\imhspace}
	\adjincludegraphics[height=\imheighta]{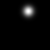}\hspace{\imhspace}
	\adjincludegraphics[height=\imheighta]{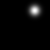}\hspace{\imhspace}
	\adjincludegraphics[height=\imheighta]{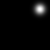}\hspace{\imhspace}
	\adjincludegraphics[height=\imheighta]{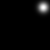}\hspace{\imhspace}
	\adjincludegraphics[height=\imheighta]{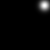} \\
	\vspace{-1.5em}
	\hspace{0.em} Horizontal \hspace{1.25em} Vertical \hfill \vspace{0.5em} \\
	\adjincludegraphics[height=\imheightb,trim={{0.\width} {0.02\height} {0.0\width} {0.06\height}}, clip]{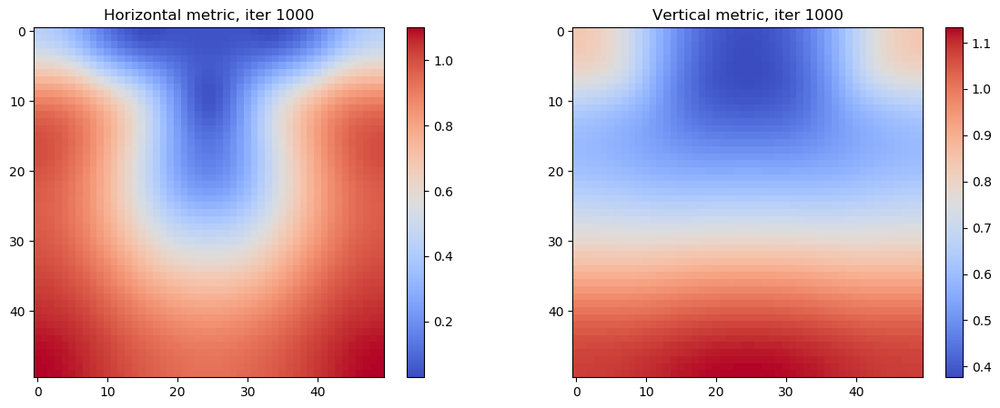}\hspace{\imhspace}
	\adjincludegraphics[height=\imheightb]{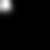}\hspace{\imhspace}
	\adjincludegraphics[height=\imheightb]{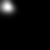}\hspace{\imhspace}
	\adjincludegraphics[height=\imheightb]{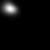}\hspace{\imhspace}
	\adjincludegraphics[height=\imheightb]{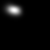}\hspace{\imhspace}
	\adjincludegraphics[height=\imheightb]{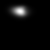}\hspace{\imhspace}
	\adjincludegraphics[height=\imheightb]{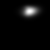}\hspace{\imhspace}
	\adjincludegraphics[height=\imheightb]{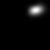}\hspace{\imhspace}
	\adjincludegraphics[height=\imheightb]{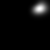}\hspace{\imhspace}
	\adjincludegraphics[height=\imheightb]{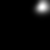}\hspace{\imhspace}
	\adjincludegraphics[height=\imheightb]{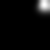} \\

	\caption{\new{Example of the resulting metric on a dataset that is not generated by the forward model.
	First row: synthetic dataset of a Gaussian distribution following the same trajectory as in \autoref{fig:toyinverseprob3}.
	Second row: we apply our algorithm to this dataset to recover a metric (two leftmost images), and the reconstructions.
	We recover a metric that shares the main features of the one in \autoref{fig:toyinverseprob3}.}
	}
	\label{fig:toydataMC}
\end{figure*}

\new{
\subsubsection{Multiple sequences}

We further attempted to recover metrics from multiple observed paths by summing our functional \eqref{eq:mainpbmdisc} over multiple input sequences.
In \autoref{fig:toy_multi}, we show a ground truth metric, and four sequences generated with that metric.
They are all interpolations of two Gaussian distributions placed on either side of the support: horizontally, vertically, or diagonally.
We then present in that figure three different metrics: the first is learned on the first sequence, the second is learned on the third sequence, and the last one is learned on all sequences.
We can see that the metric learned on all sequences is closer to the ground truth.
By providing more learning data, we reduce the space of solutions and prevent overfitting, effectively making the problem better posed.
This comes at the cost of computational overhead, but that only increases linearly with the number of sequences.
Parameters for these experiments were $n=50$, $\sclr{L}=50$, $\varepsilon$=1.2e-2, $\sclr{S}=100$, $\lambda_c=0$, $\lambda_s=3$, and an $\mathcal{L}_2$ loss.
}

\begin{figure}
	\centering
	
	Horizontal Metric \hspace{5em} Vertical Metric \vspace{0.5em} \\
	\raisebox{2.7em}{\rotatebox{90}{Ground truth}}\hspace{0.5em}
	\adjincludegraphics[width=0.93\linewidth,trim={{0.1\width} {0\height} {0.1\width} {0.11\height}}, clip]{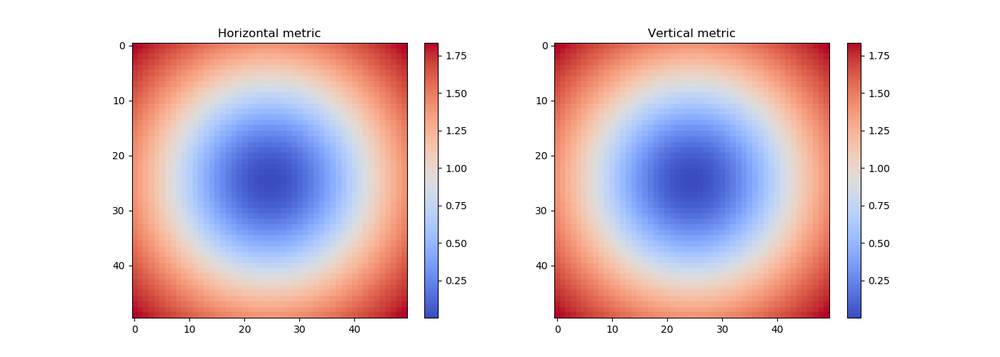} \\
	\vspace{-1em}
	
	\raisebox{1em}{\rotatebox{90}{S4\hspace{2em}S3\hspace{2em}S2\hspace{2em}S1\hspace{2em}}}\hspace{1em}
	\adjincludegraphics[width=0.89\linewidth]{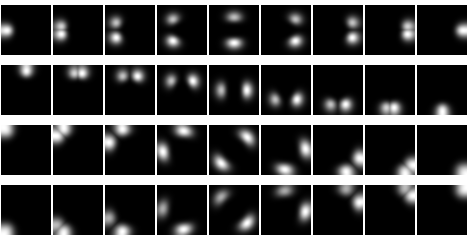} \\
	\vspace{1em}
	
	\raisebox{2.5em}{\rotatebox{90}{Learned on S1}}\hspace{0.5em}
	\adjincludegraphics[width=0.93\linewidth,trim={{0.1\width} {0\height} {0.1\width} {0.11\height}}, clip]{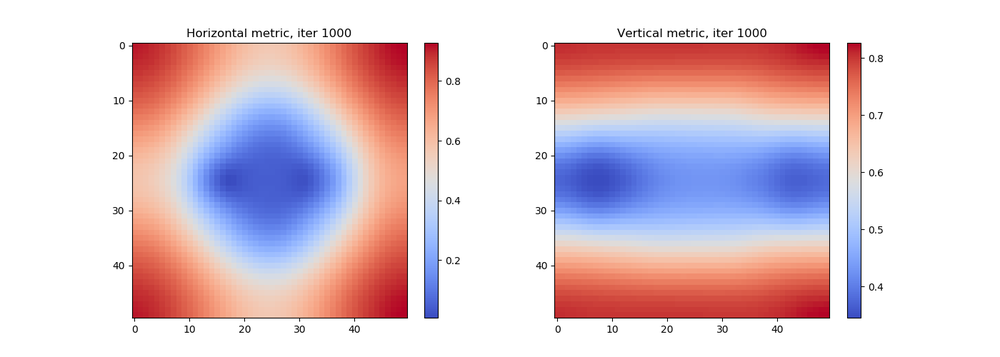} \\
	\raisebox{2.5em}{\rotatebox{90}{Learned on S3}}\hspace{0.5em}
	\adjincludegraphics[width=0.93\linewidth,trim={{0.1\width} {0\height} {0.1\width} {0.11\height}}, clip]{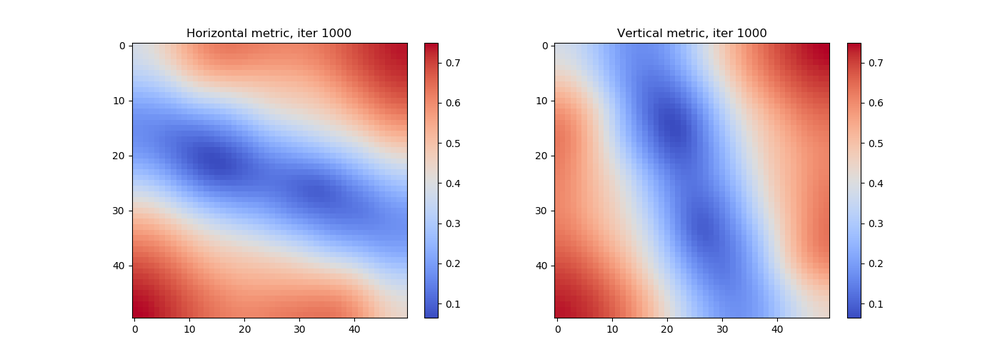} \\
	\raisebox{2.2em}{\rotatebox{90}{Learned on all}}\hspace{0.5em}
	\adjincludegraphics[width=0.93\linewidth,trim={{0\width} {0\height} {0\width} {0.06\height}}, clip]{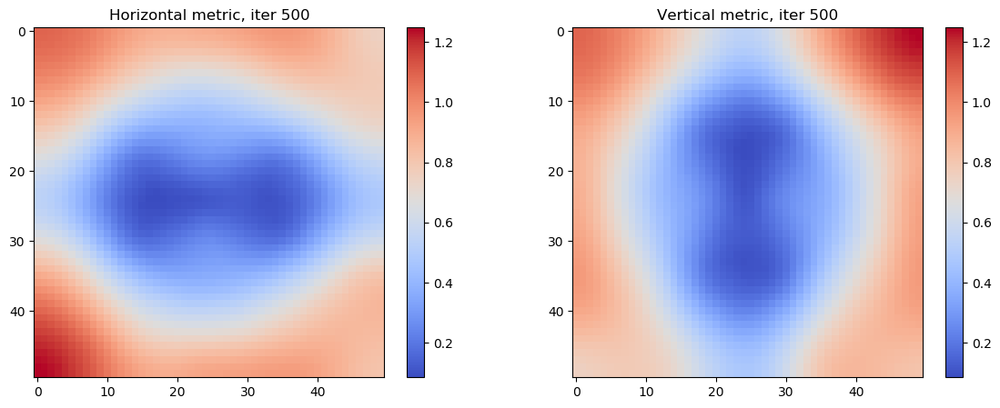}
	\caption{\new{Preliminary experiment of learning on multiple sequences. Using the ground truth metric in the first row, we generate 4 sequences of moving Gaussian distributions. In the last 3 rows, we compare the metric learned: on the first sequence, on the third sequence and on all sequences. Learning on multiple sequences makes the problem better posed, and we see that the recovered metric is closer to the ground truth than when learning with only one sequence.}
	}
	\label{fig:toy_multi}
\end{figure}

\subsection{Evaluation}
\label{sec:evaluation}

\subsubsection{Diffusion equation}
\label{sec:eval_diff_eq}

The parameters $\varepsilon$ and $\sclr{S}$ need to be carefully set for solving the diffusion equation.
Indeed, depending on their value, the formula \eqref{eq:applykernelMk} yields a kernel that is a better or worse approximation of the heat kernel, which directly impacts the accuracy of the displacement interpolations computed with it.
We demonstrate these effects in 2-D, by interpolating between two Dirac masses across a 50x50 image.
We plot the middle slice of the 2-D image as a 1-D function, for easier visualization.
In \autoref{fig:eval_interp}, we plot 10 steps of an interpolation in each subplot, for different values of $\varepsilon$ and $\sclr{S}$, with a Euclidean metric (all metric weights equal to 1).

We observe a trade-off between having sharp interpolations, and having evenly spaced interpolants, which means a constant-speed interpolation.
It is important to note that memory footprint grows almost linearly with $\sclr{S}$ (see next paragraph), since every intermediate vector in \eqref{eq:applykernelMk} is stored for the backward pass.
In practice, we use either $\varepsilon$=4e-2 and $\sclr{S}=20$, or $\varepsilon$=1.2e-3 and $\sclr{S}=50$.
With this level of smoothing, we set the number of Sinkhorn iterations to 50, which is generally enough for the Sinkhorn algorithm to converge. 

\begin{figure}
	\centering
	
	\newcommand{\toplabelspace}{3em}
	\hspace{3em}$\sclr{S}=10$\hspace{\toplabelspace}$\sclr{S}=20$\hspace{\toplabelspace}$\sclr{S}=50$ \hspace{\toplabelspace}$\sclr{S}=100$ 
	\vspace{0.5em}\\
	\rotatebox{90}{\hspace{0.3em}$\varepsilon$=4e-2\hspace{0.6em}$\varepsilon$=1.2e-2\hspace{0.6em}$\varepsilon$=4e-3} %
	\hspace{0.2em}
	\adjincludegraphics[width=0.92\linewidth,trim={{0.\width} {0\height} {0.0\width} {0.25\height}}, clip]{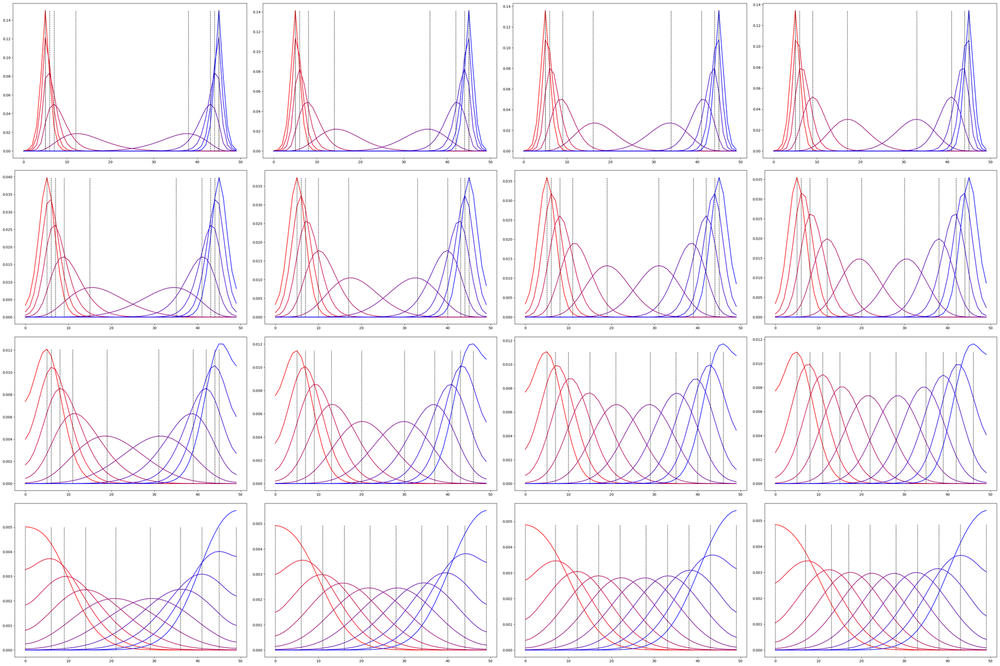}
	
	\caption{Influence of parameters $\varepsilon$ (diffusion time) and $\sclr{S}$ (number of time discretization sub-steps) on displacement interpolation between 2 Dirac masses, computed with 50 Sinkhorn iterations.
		Each plotted line is the 1-D middle slice of a 2-D image.
		Vertical dotted line are also drawn at the maximum of each interpolant, to better visualize their spacing.
		We notice that there is a trade-off between the smoothness of interpolation, and the spacing equality between interpolants.
		An equal spacing translates a constant speed interpolation.}
	\label{fig:eval_interp}
\end{figure}

\subsubsection{Regularization}
\label{sec:eval_regularization}

In order to evaluate the influence of the regularization, we compare the same experiment (the one conducted in \autoref{fig:toyinverseprob2}), with one of the two regularizers ($f_c$ and $f_s$).
The first regularizer $f_c$ effectively stabilizes the values around 1, but the recovered metric is noisy, with patterns that reflect over-fitting.
The second regularizer $f_s$ effectively produces a smooth metric, but we note that the metric values have drawn away from their initial value of 1.
After experimenting with each one, we observed that while reconstruction errors are smaller with $f_c$ (which is another sign of overfitting), the regularizer $f_s$ produces more interpretable results, and allows the global metric scale to shift in order to adapt to the input sequence.
Moreover, combining both generally does not significantly change the result compared to having only $f_s$.
Finally, tuning the $\lambda_s$ parameter allows the user to specify the desired smoothing scale (max spatial frequency) in the final metric.

\begin{figure}
	\centering
	\newcommand{\imhspace}{\dimexpr 0pt}
	\newcommand{\imvspace}{\dimexpr -1pt + 0pt}
	\newcommand{\imwidth}{\dimexpr(0.9\linewidth-\imhspace*0)/1}
	
	Horizontal Metric \hspace{5em} Vertical Metric \vspace{0.5em}\\
	
	\raisebox{2em}{\rotatebox{90}{Regularizer $f_c$}}\hspace{0.5em}
	\adjincludegraphics[width=\imwidth,trim={{0.1\width} {0\height} {0.1\width} {0.11\height}}, clip]{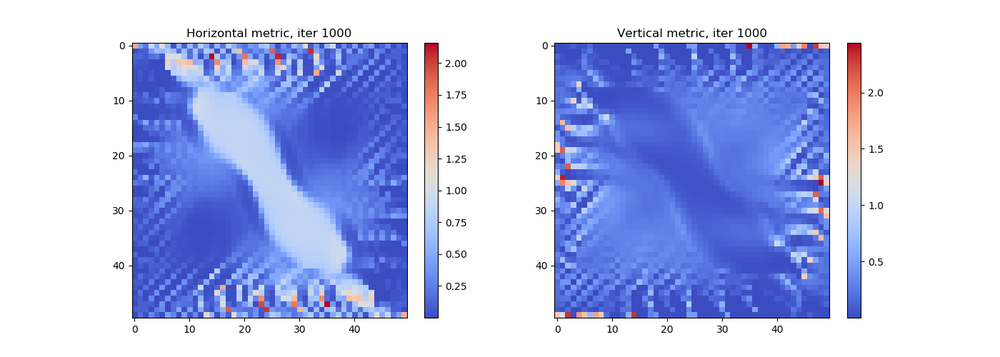} \\
	\vspace{\imvspace}
	\raisebox{2em}{\rotatebox{90}{Regularizer $f_s$}}\hspace{0.5em}
	\adjincludegraphics[width=\imwidth,trim={{0.1\width} {0\height} {0.1\width} {0.11\height}}, clip]{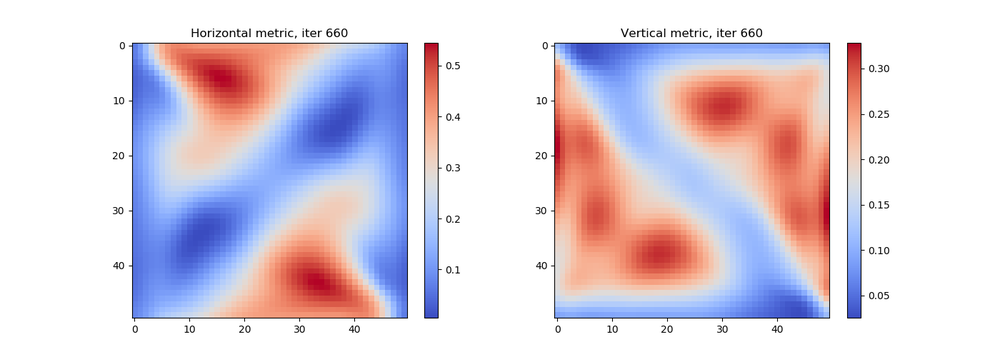} \\

	\caption{In this experiment, we show the effects of each regularizer ($f_c$ and $f_s$) on the final metric, using the experiment presented in \autoref{fig:toyinverseprob2}.
	First row: metric obtained with $\lambda_c = 0.03, \lambda_s = 0$.
	Second row: metric obtained with $\lambda_c = 0, \lambda_s = 0.03$.
	$f_c$ constrains the weights to be close to 1, while $f_s$ constrains them to be spatially smooth.}
	\label{fig:influence_regul}
\end{figure}

\subsubsection{Initialization}
\label{sec:eval_init}

Since the problem we are addressing is non-convex, the initialization of the metric weights is expected to have non-negligible effects on the final result.
In \autoref{fig:influence_init}, we present the end metric of the experiment in \autoref{fig:toyinverseprob2} with $\lambda_s=0.3$, and for three different initializations: (1) constant initialization to 1, (2) random initialization in [0.3,3] uniformly in log scale, and (3) random initialization in [0.1,10] uniformly in log scale.
We observe that the level of noise in (2) does not change the result significantly, but the one in (3) does.
In (2), the initial noise did not impact the final result, because it has been smoothed out by the regularization.
We conclude that the algorithm allows for some noise in the initialization, but a noise level that is too high cannot be smoothed out by the regularizer, and impacts the reconstruction and the final metric significantly.

\begin{figure}
	\centering
	\newcommand{\imhspace}{\dimexpr 0pt}
	\newcommand{\imvspace}{\dimexpr -1pt + 0pt}
	\newcommand{\imwidth}{\dimexpr(0.9\linewidth-\imhspace*0)/1}

	Horizontal Metric \hspace{5em} Vertical Metric \vspace{0.5em}\\
	
	\raisebox{2em}{\rotatebox{90}{Constant = 1}}\hspace{0.5em}
	\adjincludegraphics[width=\imwidth,trim={{0.1\width} {0\height} {0.1\width} {0.11\height}}, clip]{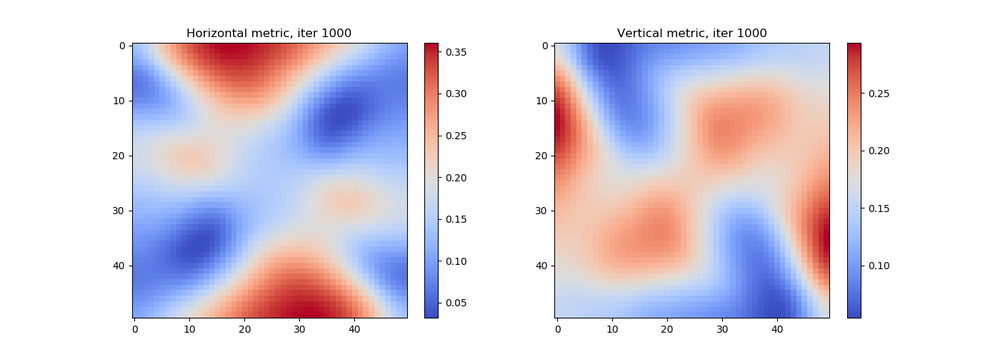} \\
	\vspace{\imvspace}
	\raisebox{1em}{\rotatebox{90}{Random in [0.3,3]}}\hspace{0.5em}
	\adjincludegraphics[width=\imwidth,trim={{0.0\width} {0.0\height} {0.0\width} {0.06\height}}, clip]{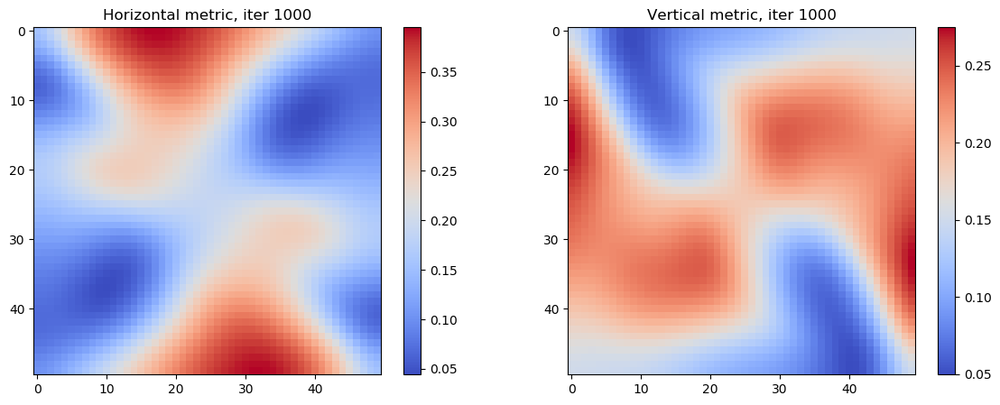} \\
	\vspace{\imvspace}
	\raisebox{1em}{\rotatebox{90}{Random in [0.1,10]}}\hspace{0.5em}
	\adjincludegraphics[width=\imwidth,trim={{0.0\width} {0.0\height} {0.0\width} {0.06\height}}, clip]{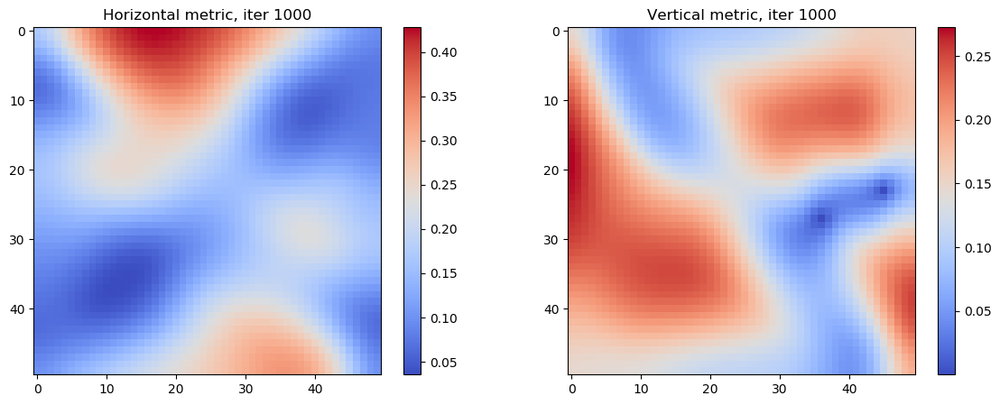} \\
	\vspace{\imvspace}

	\caption{We present the final metric of the experiment in \autoref{fig:toyinverseprob2} for 3 different initializations: constant with weights equal to 1, random in log space in [0.3,3], and random in log space in [0.1,10].
	The algorithm is robust to a half order of magnitude in the metric weights, but not to a full one.}
	\label{fig:influence_init}
\end{figure}

\subsubsection{Loss function}
\label{sec:eval_loss}

The choice of the loss function $\mathcal{L}$ is left to the user, depending on what works best with their application.
In \autoref{fig:influence_loss}, we show three 2-D metrics learned on the synthetic experiment described in \autoref{fig:toyinverseprob2}, using the different loss functions \eqref{eq:lossL1},\eqref{eq:lossL2} and \eqref{eq:lossKL}

\begin{figure}
	\centering
	\newcommand{\imhspace}{\dimexpr 0pt}
	\newcommand{\imvspace}{\dimexpr -1pt + 0pt}
	\newcommand{\imwidth}{\dimexpr(0.93\linewidth-\imhspace*0)/1}
	
	Horizontal Metric \hspace{5em} Vertical Metric \vspace{0.5em}\\
	\raisebox{4em}{\rotatebox{90}{$\mathcal{L}_{2}$ loss}}\hspace{0.5em}
	\adjincludegraphics[width=\imwidth,trim={{0.1\width} {0\height} {0.1\width} {0.11\height}}, clip]{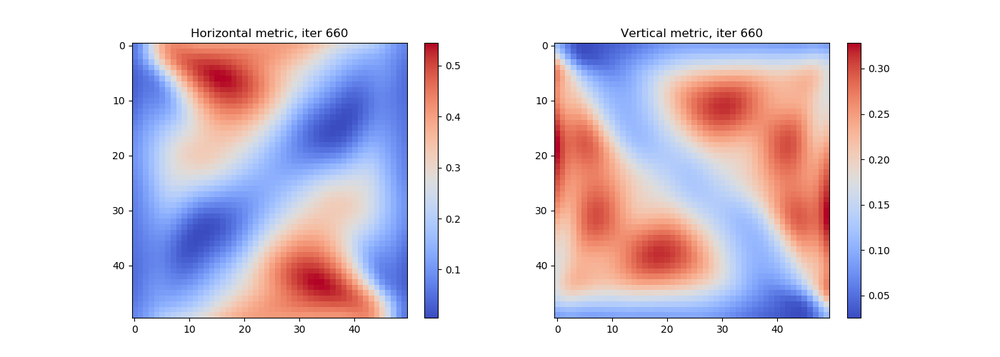} \\
	\vspace{\imvspace}
	\raisebox{4em}{\rotatebox{90}{$\mathcal{L}_{1}$ loss}}\hspace{0.5em}
	\adjincludegraphics[width=\imwidth,trim={{0.1\width} {0\height} {0.1\width} {0.11\height}}, clip]{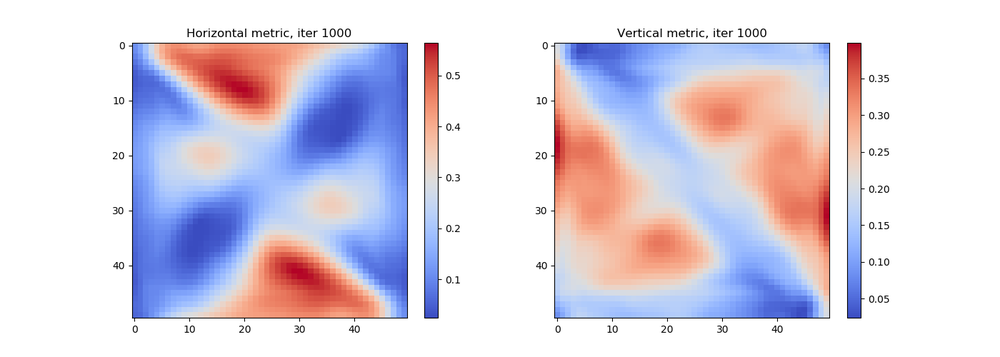} \\
	\vspace{\imvspace}
	\raisebox{4em}{\rotatebox{90}{$\mathcal{L}_{KL}$ loss}}\hspace{0.5em}
	\adjincludegraphics[width=\imwidth,trim={{0.1\width} {0\height} {0.1\width} {0.11\height}}, clip]{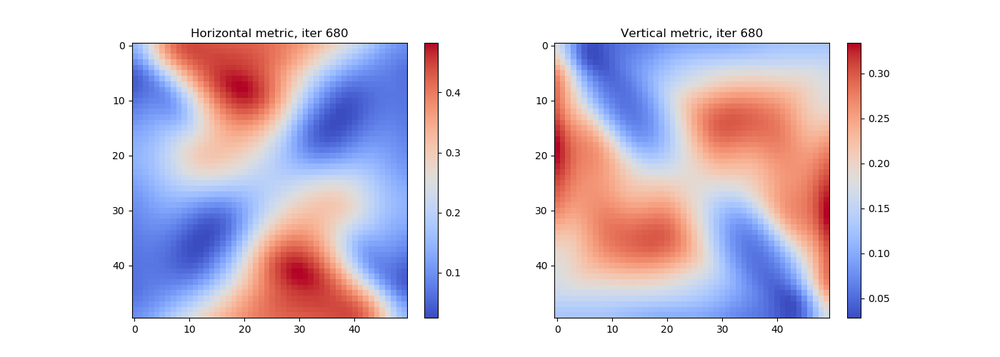} \\

	\caption{The loss function $\mathcal{L}$ influences the resulting metric.
	We present the metric learned during three experiments using the same parameters, but three different loss functions $\mathcal{L}_{2}$, $\mathcal{L}_{1}$ and $\mathcal{L}_{KL}$.
	The experiment is the one described in \autoref{fig:toyinverseprob2}.
	One must choose the loss function depending on the application.}
	\label{fig:influence_loss}
\end{figure}

\subsubsection{Timing and memory}
\label{sec:eval_time_mem}

In Table \ref{tab:timings}, we give the time and memory requirements of our algorithm, depending on the problem parameters.
We use the same type of measures as previously described, that is, those defined on graphs that are $\sclr{d}$-dimensional Cartesian grids.
The problem size is therefore $N=n^d$ and $\sclr{S}$ is the number of sub-steps to solve the heat equation.
The entropic regularization factor $\varepsilon$ (which is used here as a diffusion time) does not affect the runtime.
We give the timings for 500 L-BFGS iterations, which in our use cases, was generally sufficient for the algorithm to converge.

\begin{table}
	\begin{tabular}{cccc|ccc}
		$d$ & $n$ & $N$ & $\sclr{S}$ & $t_{500}(h)$ & Mem.(GB) & Threads \\
		\hline
		2 & 50 & 2500 & 20  & 1.6 & 1.3 & 1 \\
		2 & 50 & 2500 & 100 & 7   & 4.7 & 1 \\
		2 & 100 & 10000 & 20  & 13  & 4   & 1 \\
		2 & 100 & 10000 & 100 & 60  & 16  & 1 \\
		\hline
		3 & 16 & 4096 & 20  & 9   & 1.7 & 1 \\
		3 & 16 & 4096 & 50  & 25  & 3.3 & 1 \\
		3 & 16 & 4096 & 100 & 46  & 6.2 & 1 \\
		3 & 32 & 32768 & 20 & 110 & 10.9 & 8
	\end{tabular}
	\caption{Time and memory requirements of our algorithm, with regard to problem size $N=n^d$ and $\sclr{S}$ the number of sub-steps for solving the heat equation.
	``$t_{500}$'' is the time it takes to run 500 iterations of L-BFGS, expressed in hours (h).
	``Mem.'' is the maximum resident memory that the algorithm requires, and ``Threads'' is the number of threads it runs on.}
	\label{tab:timings}
\end{table}

This algorithm is difficult to parallelize because we need to solve a very large number of medium-size linear systems, which individually do not benefit from multi-threading.
Giving more than one thread to the algorithm was only faster for $N=32^3$.
If instead we parallelize over input images (we generally have around 10), the memory footprint grows 10 times, because the implementation is in Python, which duplicates memory for multi-processing.

\subsection{Learning color evolutions}
\label{sec:colors}

We now demonstrate an application of our algorithm that deals with 3-D color histograms in the RGB color space.
An important question in imaging and learning is which color space to use.
The RGB space is simple to use, but variations in that space do not reflect variations of color perceived by the human eye.
Other spaces, such as L*a*b* or L*u*v* have been designed to counteract this, and match variations in perception and space.
Learning a ground color metric is a way to automatically fit the color space to the problem under consideration.
Note that the problem of color metric learning in psychophysics has a long history, starting with the idea of MacAdam's ellipses~\cite{macadam_visual_1942}, which introduces a Riemannian metric (corresponding to ellipses) to fit perceptual thresholds.

\new{
Given an input sequence of sunset images (for example \autoref{fig:meteora2}), we compute each input's color histogram, and use our algorithm to learn the metric for which the histogram sequence resembles an optimal transport of mass.
All sequences presented hereafter contain around ten frames, but we only show five of them for brevity.
Our final goal is to create a new sunset sequence from a new pair of day/night images, by interpolating between them using the learned metric, and transferring the interpolated color histograms onto the day image.
Transferring the colors of an image (or of a color histogram) onto another image can be done via regularized OT, since the transport plan defines a mapping between the two histograms.
This transfer is performed using the \emph{barycentric projection map} followed by a bilateral filtering of the resulting image.
The reader can refer to \citet{peyre_computational_2018} and \citet{solomon_convolutional_2015} for more details.

\subsubsection{Validating a learned metric}
\label{sec:colors_validate}

Before we apply a learned metric to a new dataset, we verify if it performs well on the training dataset.
We first learn a metric from the \emph{meteora2} sequence (\autoref{fig:meteora2}), and visualize the reconstruction of the input histogram sequence, at the end of the metric learning process (\autoref{fig:met2COOI500_rec}).

We then check if the histograms reconstructed using the learned metric (\autoref{fig:met2COOI500_rec} second row) yield an image sequence that is close to the original one.
Therefore, instead of transferring these histograms on a new day image, we transfer them on the first frame of the same sequence.
In \autoref{fig:met2_met2_check1}, we compare this image sequence (fourth row) with two other methods of interpolation between the first and last frame, and a direct color transfer (no interpolation).
The first row is the \emph{meteora2} dataset (which is the ground truth), the second row is computed with a linear interpolation, and the third row is computed through a displacement interpolation with a Euclidean metric.
Finally, the last row shows a direct color transfer of each frame of \emph{meteora2} on its first frame.

Even though the differences are subtle, we can see that with our method of interpolation (using the learned metric), the colors (especially the red/orange tone) are closer to the ground truth and the direct transfer, than with a linear or Euclidean OT interpolation.

}

\subsubsection{Reusing the learned metric}
\label{sec:colors_newsunset}

We now create a new sunset sequence from a pair of new day/night images, as described earlier.
The image pair is extracted from the \emph{country1} dataset (\autoref{fig:quincy1}), where we take the first and the last frame.
We first learn a metric on the \emph{seldovia2} dataset (\autoref{fig:seldovia2}), with histograms of size $16^3$, the $\mathcal{L}_2$ loss, 50 Sinkhorn iterations, 500 L-BFGS iterations, an entropic regularization of $\varepsilon = 4e-3$, $\sclr{S}=20$ sub-steps, and a metric regularizer parameter $\lambda_s=1$.
Next, we interpolate between the day and night histograms of the \emph{country1} dataset, using the learned metric, which is upsampled to $31^3$ in order to decrease color quantization errors.
Finally, we transfer each interpolated color histograms on the day frame to reproduce a sunset sequence.

\new{
In \autoref{fig:quincy1_sel2F0016I500_hists}, we compare the histogram sequence interpolated using the learned metric (3rd row), with a linear interpolation (1st row), and a displacement interpolation with a Euclidean metric (2nd row).
In rows 2-4 of \autoref{fig:quincy1_sel2F0016I500_images}, we show the color transfer of each interpolated histogram of \autoref{fig:quincy1_sel2F0016I500_hists} onto the day frame.
In the first row, we show the \emph{country1} dataset, which constitutes a ground truth, and in row 5, we compare the results with a direct transfer of the \emph{seldovia2} dataset on the day image of the \emph{country1} dataset.

With linear interpolation, mass does not move. This results in color artifacts in our examples, such as clouds remaining bright and white until the very end of the sequence.

With the displacement interpolation with a Euclidean metric, mass travels in straight line between the first and last frame.
We can see on the histogram sequence that mass travels close to the diagonal of the histogram, which represents the grey levels.
Therefore, the resulting image sequence shows no red/orange tones that are typical of a sunset.

With our method, mass does not travel in straight line, by virtue of the learned metric. 
We can see on the histogram sequence that mass bypasses the diagonal of greys, and therefore travels in the red/orange tones, which results in an image sequence that is much closer to the ground truth.
However, we also notice that the mass is split into two packages, one going through the red region, and the other one through the blue region, resulting in the sky getting bluer in the image sequence.
This is an artifact that can be explained by the relatively high diffusion level required to have equally spaced interpolations, as explained in \autoref{fig:eval_interp}.

}

Finally, a direct transfer also gives a plausible sunset sequence, however, the original colors of the target dataset (\emph{country1}) are not preserved.
Moreover, our method allows interpolation with an arbitrary number of frames, whereas the direct transfer can only produce the number of frames available in the source dataset.

\begin{figure*}
	\centering
	\newcommand{\imhspace}{\dimexpr 3pt}
	\newcommand{\imvspace}{\dimexpr -1pt + 3pt}
	\newcommand{\imwidth}{\dimexpr(0.99\linewidth-\imhspace*4)/5}

	\adjincludegraphics[width=\imwidth]{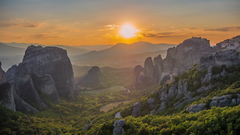}\hspace{\imhspace}
	\adjincludegraphics[width=\imwidth]{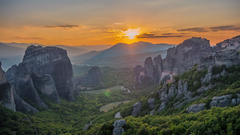}\hspace{\imhspace}
	\adjincludegraphics[width=\imwidth]{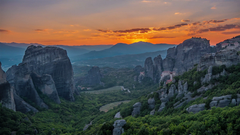}\hspace{\imhspace}
	\adjincludegraphics[width=\imwidth]{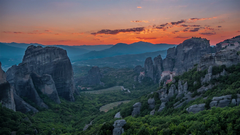}\hspace{\imhspace}
	\adjincludegraphics[width=\imwidth]{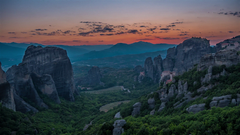} \\
	\vspace{\imvspace}
	\adjincludegraphics[width=\imwidth,trim={{0.125\width} {0\height} {0\width} {0.098\height}}, clip]{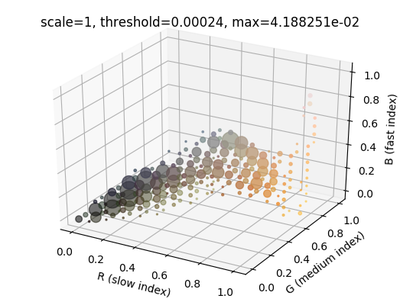}\hspace{\imhspace}
	\adjincludegraphics[width=\imwidth,trim={{0.125\width} {0\height} {0\width} {0.098\height}}, clip]{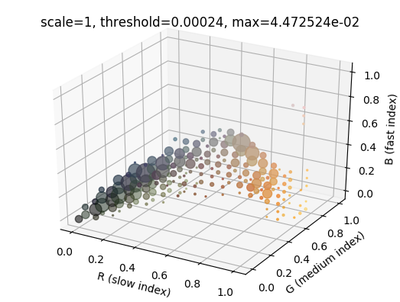}\hspace{\imhspace}
	\adjincludegraphics[width=\imwidth,trim={{0.125\width} {0\height} {0\width} {0.098\height}}, clip]{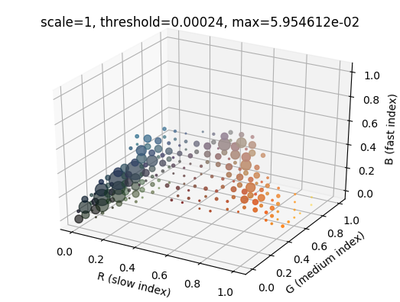}\hspace{\imhspace}
	\adjincludegraphics[width=\imwidth,trim={{0.125\width} {0\height} {0\width} {0.098\height}}, clip]{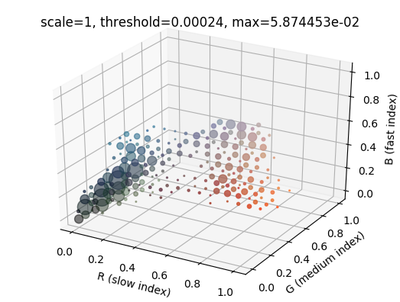}\hspace{\imhspace}
	\adjincludegraphics[width=\imwidth,trim={{0.125\width} {0\height} {0\width} {0.098\height}}, clip]{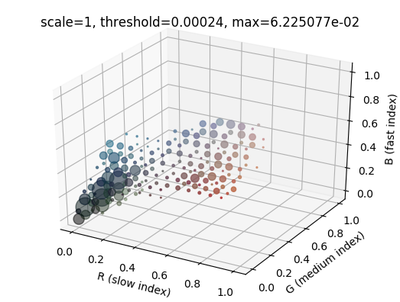} \\

	\caption{The \emph{meteora2} dataset: images (first row) and color histograms (second row).
	Video courtesy of \href{https://www.youtube.com/watch?v=tijJCrDh860}{PG ViSUAL}}
	\label{fig:meteora2}
\end{figure*}

\begin{figure*}
	\centering
	\newcommand{\imhspace}{\dimexpr 0pt}
	\newcommand{\imvspace}{\dimexpr -1pt + 0pt}
	\newcommand{\imwidth}{\dimexpr(0.99\linewidth-\imhspace*4)/5}

	\vspace{\imvspace}
	\adjincludegraphics[width=\imwidth,trim={{0.125\width} {0\height} {0\width} {0.098\height}}, clip]{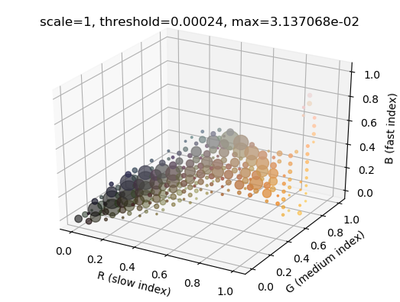}\hspace{\imhspace}
	\adjincludegraphics[width=\imwidth,trim={{0.125\width} {0\height} {0\width} {0.098\height}}, clip]{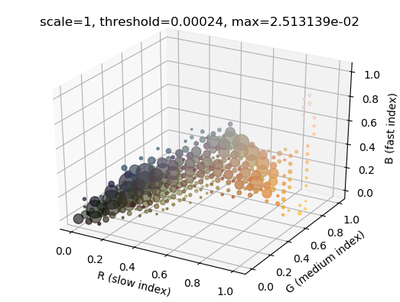}\hspace{\imhspace}
	\adjincludegraphics[width=\imwidth,trim={{0.125\width} {0\height} {0\width} {0.098\height}}, clip]{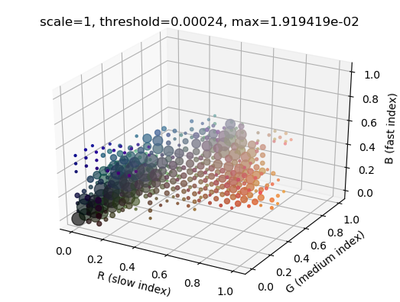}\hspace{\imhspace}
	\adjincludegraphics[width=\imwidth,trim={{0.125\width} {0\height} {0\width} {0.098\height}}, clip]{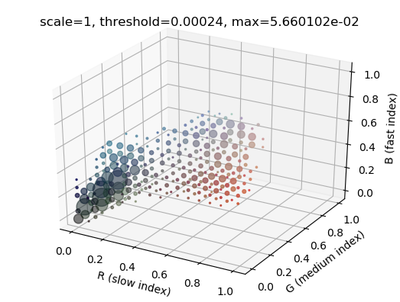}\hspace{\imhspace}
	\adjincludegraphics[width=\imwidth,trim={{0.125\width} {0\height} {0\width} {0.098\height}}, clip]{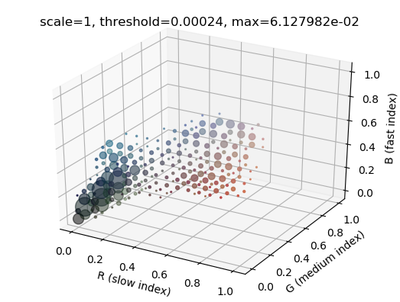} \\

	\caption{Reconstruction of the color histogram sequence of the \emph{meteora2} dataset (\autoref{fig:meteora2}), using the metric learned from it, after 500 iterations of our algorithm. 
	The ground truth is in the second row of \autoref{fig:meteora2}.
	We notice that the reconstructions are more diffuse than the inputs, due to the entropic regularization.
	}
	\label{fig:met2COOI500_rec}
\end{figure*}

\begin{figure*}
	\centering
	\newcommand{\imhspace}{\dimexpr 3pt}
	\newcommand{\imvspace}{\dimexpr -1pt + 3pt}
	\newcommand{\imwidth}{\dimexpr(0.95\linewidth-\imhspace*4)/5}
	
	\raisebox{-0.25em}{\rotatebox{90}{Ground Truth}}\hspace{0.5em}
	\adjincludegraphics[width=\imwidth]{figures/color_transfers/meteora2/red25-video_000.png}\hspace{\imhspace}
	\adjincludegraphics[width=\imwidth]{figures/color_transfers/meteora2/red25-video_002.png}\hspace{\imhspace}
	\adjincludegraphics[width=\imwidth]{figures/color_transfers/meteora2/red25-video_004.png}\hspace{\imhspace}
	\adjincludegraphics[width=\imwidth]{figures/color_transfers/meteora2/red25-video_006.png}\hspace{\imhspace}
	\adjincludegraphics[width=\imwidth]{figures/color_transfers/meteora2/red25-video_008.png} \\
	\vspace{4\imvspace}
	\raisebox{1.5em}{\rotatebox{90}{Linear}}\hspace{0.5em}
	\adjincludegraphics[width=\imwidth]{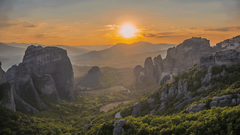}\hspace{\imhspace}
	\adjincludegraphics[width=\imwidth]{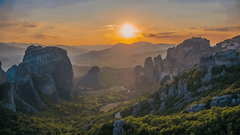}\hspace{\imhspace}
	\adjincludegraphics[width=\imwidth]{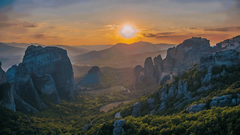}\hspace{\imhspace}
	\adjincludegraphics[width=\imwidth]{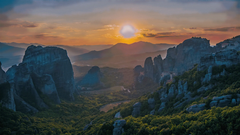}\hspace{\imhspace}
	\adjincludegraphics[width=\imwidth]{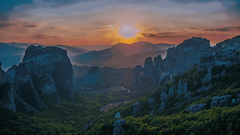} \\
	\vspace{\imvspace}
	\raisebox{0em}{\rotatebox{90}{Euclidean OT}}\hspace{0.5em}
	\adjincludegraphics[width=\imwidth]{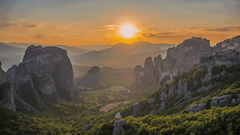}\hspace{\imhspace}
	\adjincludegraphics[width=\imwidth]{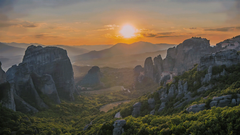}\hspace{\imhspace}
	\adjincludegraphics[width=\imwidth]{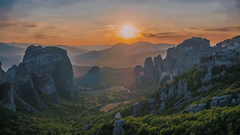}\hspace{\imhspace}
	\adjincludegraphics[width=\imwidth]{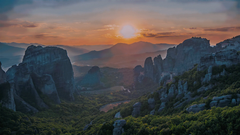}\hspace{\imhspace}
	\adjincludegraphics[width=\imwidth]{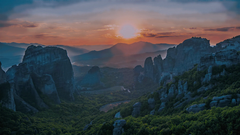} \\
	\vspace{\imvspace}
	\raisebox{0em}{\rotatebox{90}{Our method}}\hspace{0.5em}
	\adjincludegraphics[width=\imwidth]{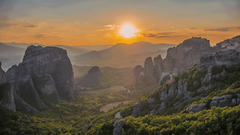}\hspace{\imhspace}
	\adjincludegraphics[width=\imwidth]{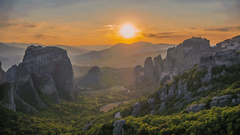}\hspace{\imhspace}
	\adjincludegraphics[width=\imwidth]{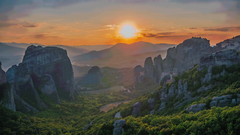}\hspace{\imhspace}
	\adjincludegraphics[width=\imwidth]{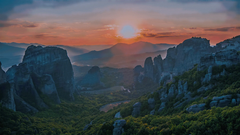}\hspace{\imhspace}
	\adjincludegraphics[width=\imwidth]{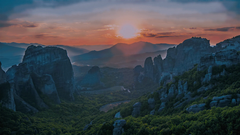} \\
	\vspace{4\imvspace}
	\raisebox{-0.25em}{\rotatebox{90}{Direct transfer}}\hspace{0.5em}
	\adjincludegraphics[width=\imwidth]{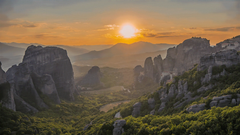}\hspace{\imhspace}
	\adjincludegraphics[width=\imwidth]{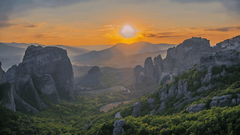}\hspace{\imhspace}
	\adjincludegraphics[width=\imwidth]{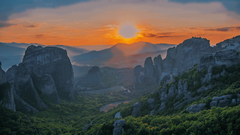}\hspace{\imhspace}
	\adjincludegraphics[width=\imwidth]{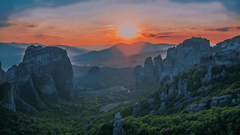}\hspace{\imhspace}
	\adjincludegraphics[width=\imwidth]{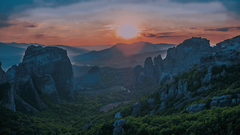} \\

	\caption{Preliminary experiment: we learn a metric on the \emph{meteora2} dataset (top row), then reinterpolate color histograms between the first and last frames using different methods, and transfer each interpolated histogram unto the first frame. 
	The interpolation methods are: linear interpolation (second row), displacement interpolation with a Euclidean metric (third row), displacement interpolation with the learned metric (fourth row). 
	The last row does not involve interpolation and is simply the color transfer of each frame of the ground truth onto the first frame.}
	\label{fig:met2_met2_check1}
\end{figure*}

\begin{figure*}
	\centering
	\newcommand{\imhspace}{\dimexpr 3pt}
	\newcommand{\imvspace}{\dimexpr -1pt + 0pt}
	\newcommand{\imwidth}{\dimexpr(0.99\linewidth-\imhspace*4)/5}
	
	\adjincludegraphics[width=\imwidth]{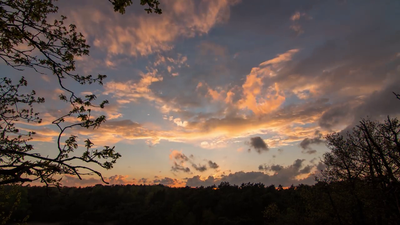}\hspace{\imhspace}
	\adjincludegraphics[width=\imwidth]{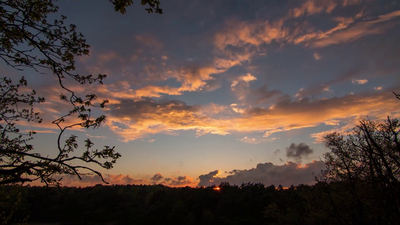}\hspace{\imhspace}
	\adjincludegraphics[width=\imwidth]{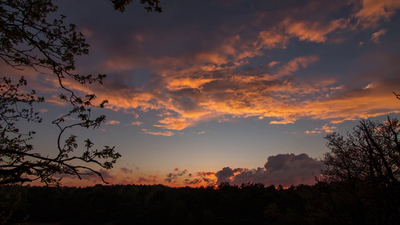}\hspace{\imhspace}
	\adjincludegraphics[width=\imwidth]{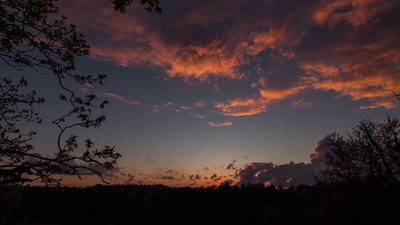}\hspace{\imhspace}
	\adjincludegraphics[width=\imwidth]{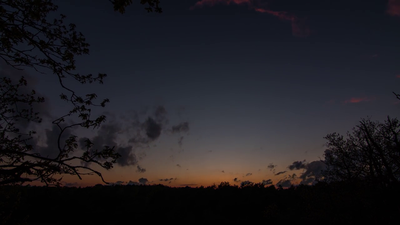} \\
	\vspace{\imvspace}
	\adjincludegraphics[width=\imwidth,trim={{0.125\width} {0\height} {0\width} {0.09\height}}, clip]{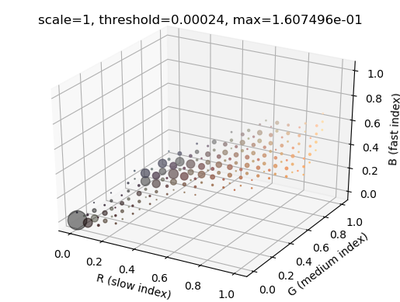}\hspace{\imhspace}
	\adjincludegraphics[width=\imwidth,trim={{0.125\width} {0\height} {0\width} {0.09\height}}, clip]{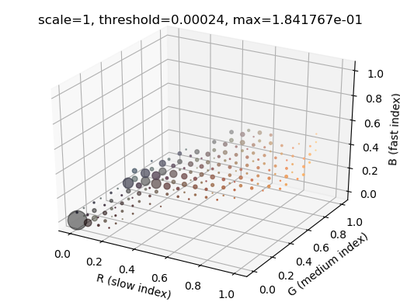}\hspace{\imhspace}
	\adjincludegraphics[width=\imwidth,trim={{0.125\width} {0\height} {0\width} {0.09\height}}, clip]{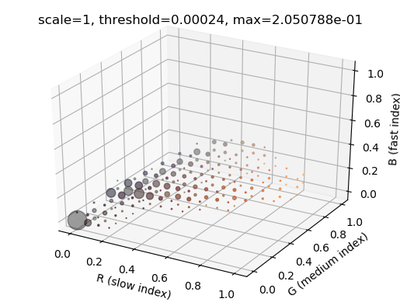}\hspace{\imhspace}
	\adjincludegraphics[width=\imwidth,trim={{0.125\width} {0\height} {0\width} {0.09\height}}, clip]{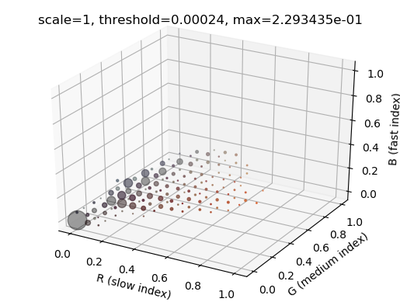}\hspace{\imhspace}
	\adjincludegraphics[width=\imwidth,trim={{0.125\width} {0\height} {0\width} {0.09\height}}, clip]{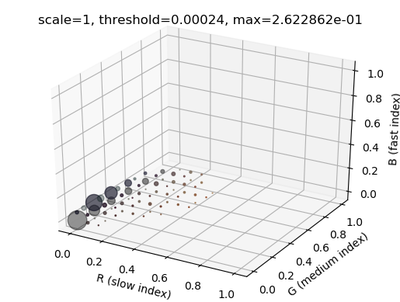} \\
	
	\caption{The \emph{country1} dataset: images (first row) and color histograms (second row).
	Video courtesy of \href{https://www.youtube.com/watch?v=cf6l7DqFzZI}{Quincy van den Boom}}
	\label{fig:quincy1}
\end{figure*}

\begin{figure*}
	\centering
	\newcommand{\imhspace}{\dimexpr 3pt}
	\newcommand{\imvspace}{\dimexpr -1pt + 0pt}
	\newcommand{\imwidth}{\dimexpr(0.99\linewidth-\imhspace*4)/5}
	
	\adjincludegraphics[width=\imwidth]{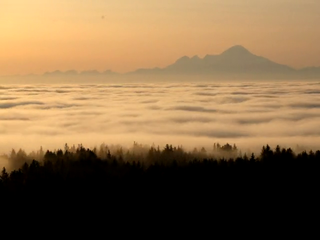}\hspace{\imhspace}
	\adjincludegraphics[width=\imwidth]{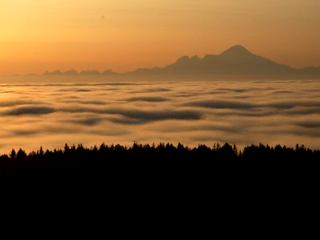}\hspace{\imhspace}
	\adjincludegraphics[width=\imwidth]{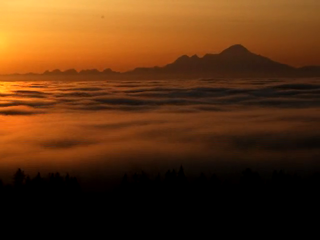}\hspace{\imhspace}
	\adjincludegraphics[width=\imwidth]{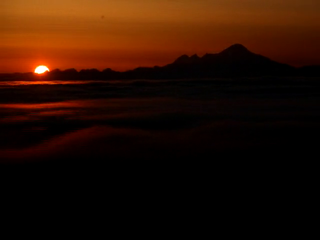}\hspace{\imhspace}
	\adjincludegraphics[width=\imwidth]{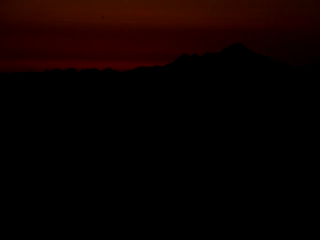} \\
	\vspace{\imvspace}
	\adjincludegraphics[width=\imwidth,trim={{0.125\width} {0\height} {0\width} {0.09\height}}, clip]{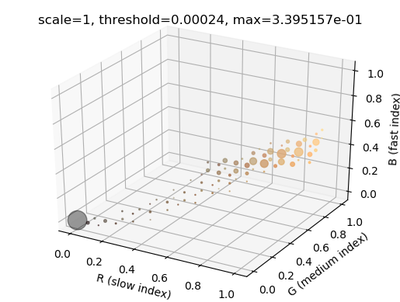}\hspace{\imhspace}
	\adjincludegraphics[width=\imwidth,trim={{0.125\width} {0\height} {0\width} {0.09\height}}, clip]{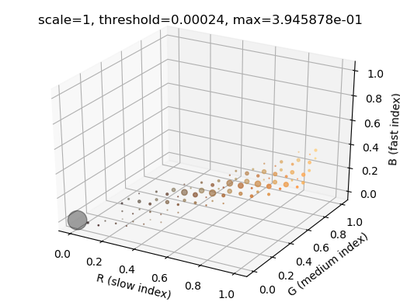}\hspace{\imhspace}
	\adjincludegraphics[width=\imwidth,trim={{0.125\width} {0\height} {0\width} {0.09\height}}, clip]{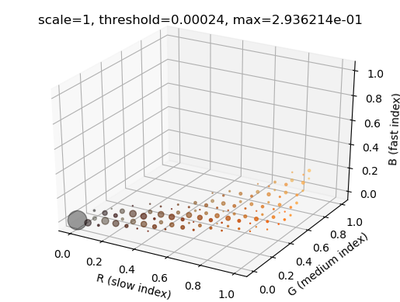}\hspace{\imhspace}
	\adjincludegraphics[width=\imwidth,trim={{0.125\width} {0\height} {0\width} {0.09\height}}, clip]{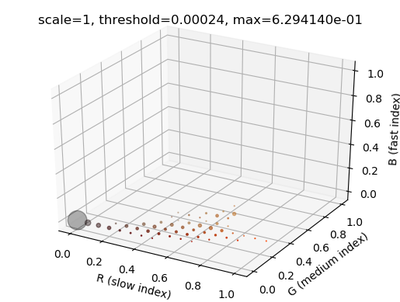}\hspace{\imhspace}
	\adjincludegraphics[width=\imwidth,trim={{0.125\width} {0\height} {0\width} {0.09\height}}, clip]{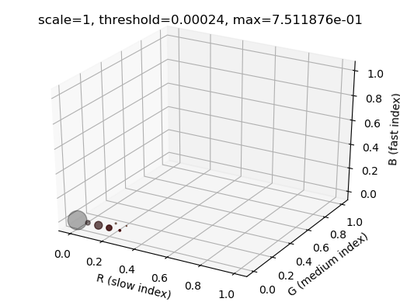} \\
	
	\caption{The \emph{seldovia2} dataset: images (first row) and color histograms (second row).
	Video courtesy of \href{https://vimeo.com/4441546}{Bretwood Higman}}
	\label{fig:seldovia2}
\end{figure*}

\begin{figure*}
	\centering
	\newcommand{\imhspace}{\dimexpr 0pt}
	\newcommand{\imvspace}{\dimexpr -1pt + 0pt}
	\newcommand{\imwidth}{\dimexpr(0.95\linewidth-\imhspace*4)/5}

	\raisebox{3em}{\rotatebox{90}{Linear}}\hspace{0.5em}
	\adjincludegraphics[width=\imwidth,trim={{0.1\width} {0.0\height} {0.0\width} {0.09\height}}, clip]{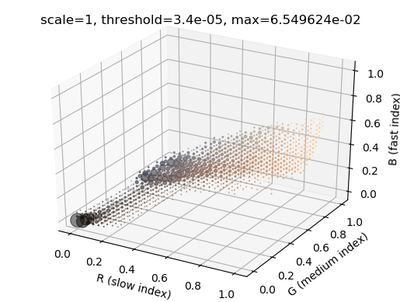}\hspace{\imhspace}
	\adjincludegraphics[width=\imwidth,trim={{0.1\width} {0.0\height} {0.0\width} {0.09\height}}, clip]{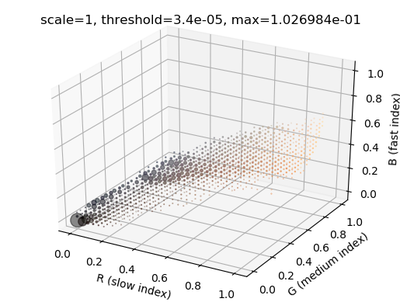}\hspace{\imhspace}
	\adjincludegraphics[width=\imwidth,trim={{0.1\width} {0.0\height} {0.0\width} {0.09\height}}, clip]{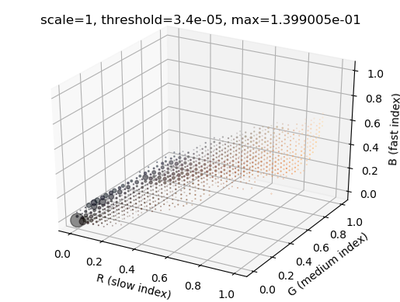}\hspace{\imhspace}
	\adjincludegraphics[width=\imwidth,trim={{0.1\width} {0.0\height} {0.0\width} {0.09\height}}, clip]{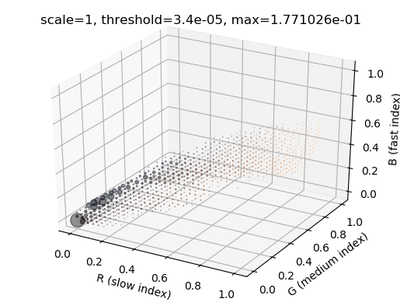}\hspace{\imhspace}
	\adjincludegraphics[width=\imwidth,trim={{0.1\width} {0.0\height} {0.0\width} {0.09\height}}, clip]{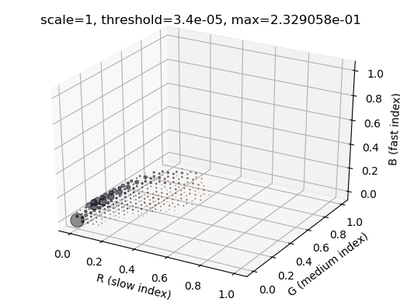} \\
	\vspace{\imvspace}
	\raisebox{1em}{\rotatebox{90}{Euclidean OT}}\hspace{0.5em}
	\adjincludegraphics[width=\imwidth,trim={{0.1\width} {0.0\height} {0.0\width} {0.09\height}}, clip]{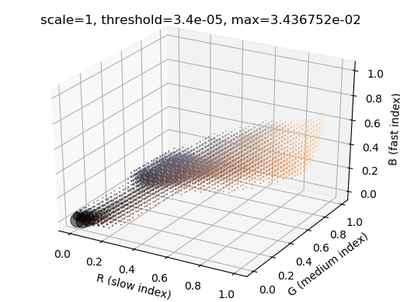}\hspace{\imhspace}
	\adjincludegraphics[width=\imwidth,trim={{0.1\width} {0.0\height} {0.0\width} {0.09\height}}, clip]{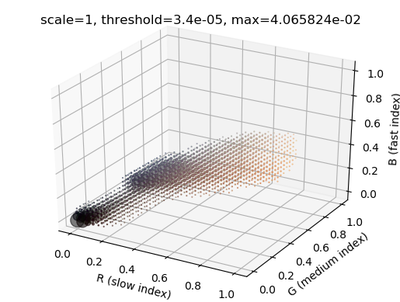}\hspace{\imhspace}
	\adjincludegraphics[width=\imwidth,trim={{0.1\width} {0.0\height} {0.0\width} {0.09\height}}, clip]{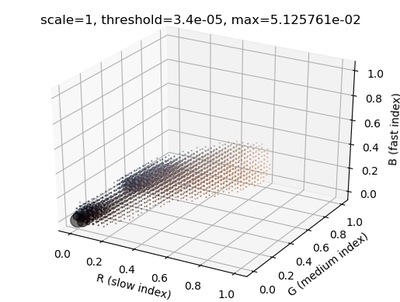}\hspace{\imhspace}
	\adjincludegraphics[width=\imwidth,trim={{0.1\width} {0.0\height} {0.0\width} {0.09\height}}, clip]{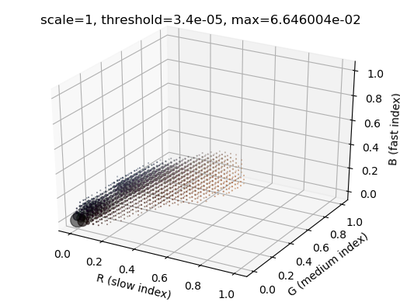}\hspace{\imhspace}
	\adjincludegraphics[width=\imwidth,trim={{0.1\width} {0.0\height} {0.0\width} {0.09\height}}, clip]{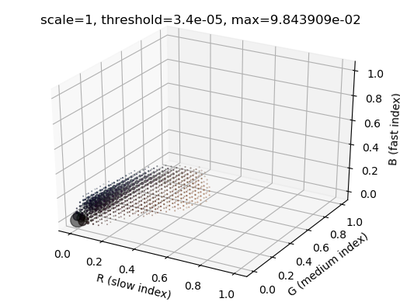} \\
	\vspace{\imvspace}
	\raisebox{1em}{\rotatebox{90}{Our method}}\hspace{0.5em}
	\adjincludegraphics[width=\imwidth,trim={{0.1\width} {0.0\height} {0.0\width} {0.09\height}}, clip]{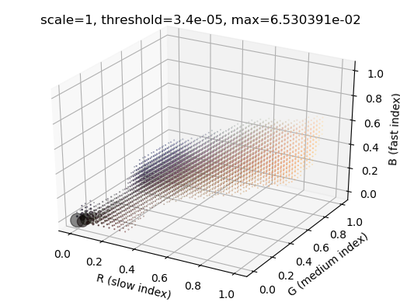}\hspace{\imhspace}
	\adjincludegraphics[width=\imwidth,trim={{0.1\width} {0.0\height} {0.0\width} {0.09\height}}, clip]{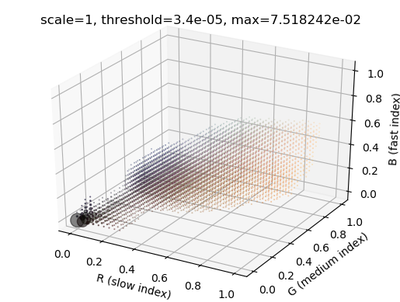}\hspace{\imhspace}
	\adjincludegraphics[width=\imwidth,trim={{0.1\width} {0.0\height} {0.0\width} {0.09\height}}, clip]{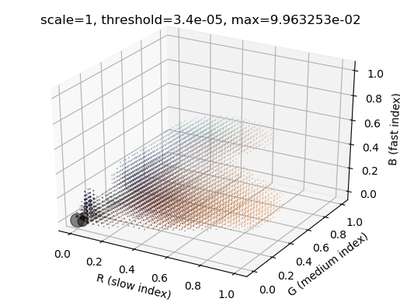}\hspace{\imhspace}
	\adjincludegraphics[width=\imwidth,trim={{0.1\width} {0.0\height} {0.0\width} {0.09\height}}, clip]{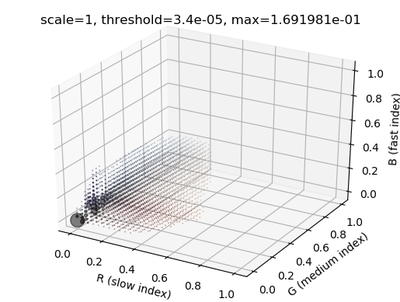}\hspace{\imhspace}
	\adjincludegraphics[width=\imwidth,trim={{0.1\width} {0.0\height} {0.0\width} {0.09\height}}, clip]{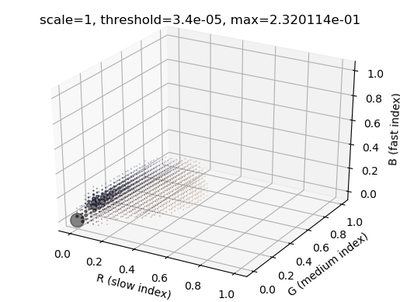} \\
	
	\caption{Interpolation between day and night histograms of the \emph{country1} dataset (\autoref{fig:quincy1}) using: linear interpolation (1st row), OT with Euclidean metric (2nd row) and OT with the metric learned on the \emph{seldovia2} (\autoref{fig:seldovia2}) dataset (3rd row).
	Color transfers using these histograms are presented in \autoref{fig:quincy1_sel2F0016I500_images}.}
	\label{fig:quincy1_sel2F0016I500_hists}

\end{figure*}

\begin{figure*}
	\centering
	\newcommand{\imhspace}{\dimexpr 2pt}
	\newcommand{\imvspace}{\dimexpr -1pt + 4pt}
	\newcommand{\imwidth}{\dimexpr(0.95\linewidth-\imhspace*4)/5}
	
	\raisebox{0em}{\rotatebox{90}{Ground truth}}\hspace{0.5em}
	\adjincludegraphics[width=\imwidth]{figures/color_transfers/quincy1/red50-video29.png}\hspace{\imhspace}
	\adjincludegraphics[width=\imwidth]{figures/color_transfers/quincy1/red50-video31.png}\hspace{\imhspace}
	\adjincludegraphics[width=\imwidth]{figures/color_transfers/quincy1/red50-video33.png}\hspace{\imhspace}
	\adjincludegraphics[width=\imwidth]{figures/color_transfers/quincy1/red50-video35.png}\hspace{\imhspace}
	\adjincludegraphics[width=\imwidth]{figures/color_transfers/quincy1/red50-video38.png} \\
	\vspace{4\imvspace}
	\raisebox{1em}{\rotatebox{90}{Linear}}\hspace{0.5em}
    \adjincludegraphics[width=\imwidth]{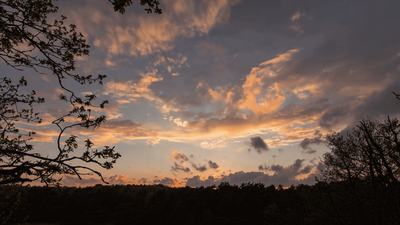}\hspace{\imhspace}
    \adjincludegraphics[width=\imwidth]{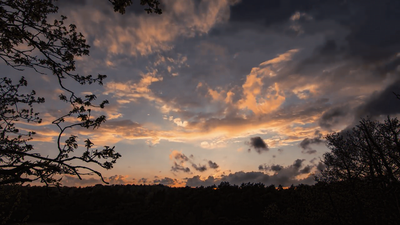}\hspace{\imhspace}
    \adjincludegraphics[width=\imwidth]{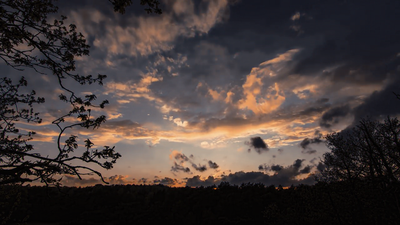}\hspace{\imhspace}
    \adjincludegraphics[width=\imwidth]{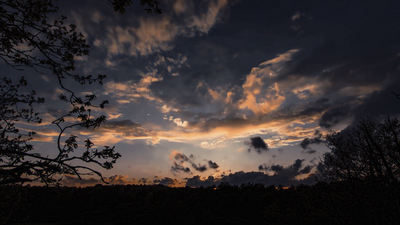}\hspace{\imhspace}
    \adjincludegraphics[width=\imwidth]{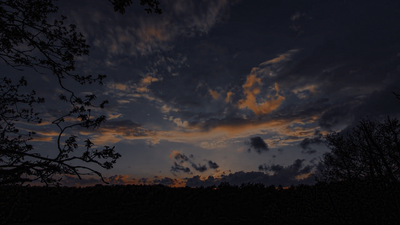} \\
    \vspace{\imvspace}
   	\raisebox{0.5em}{\rotatebox{90}{Euclid. OT}}\hspace{0.5em}
    \adjincludegraphics[width=\imwidth]{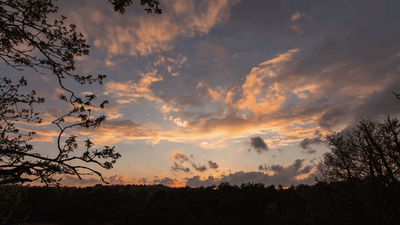}\hspace{\imhspace}
    \adjincludegraphics[width=\imwidth]{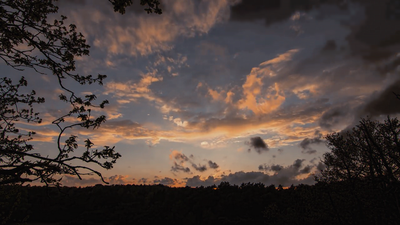}\hspace{\imhspace}
    \adjincludegraphics[width=\imwidth]{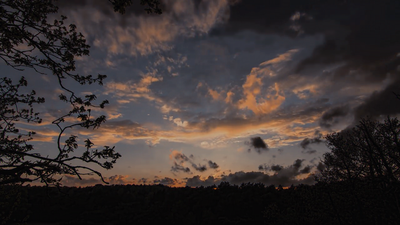}\hspace{\imhspace}
    \adjincludegraphics[width=\imwidth]{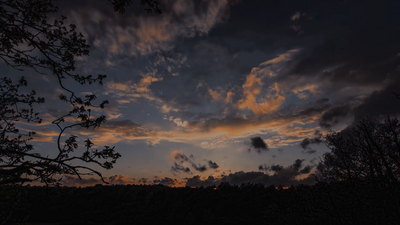}\hspace{\imhspace}
    \adjincludegraphics[width=\imwidth]{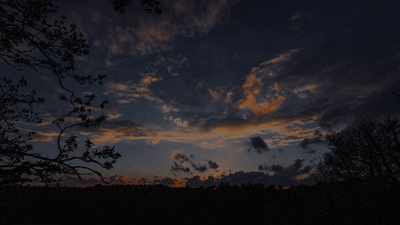} \\
    \vspace{\imvspace}
	\raisebox{0em}{\rotatebox{90}{Our method}}\hspace{0.5em}
    \adjincludegraphics[width=\imwidth]{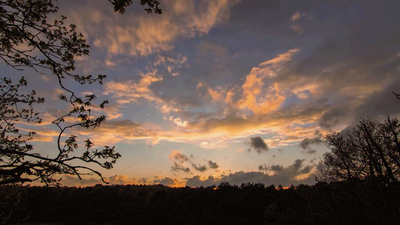}\hspace{\imhspace}
    \adjincludegraphics[width=\imwidth]{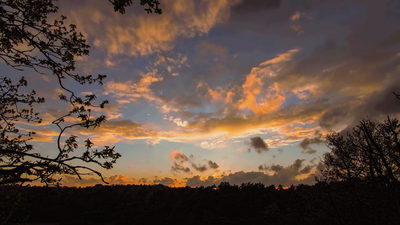}\hspace{\imhspace}
    \adjincludegraphics[width=\imwidth]{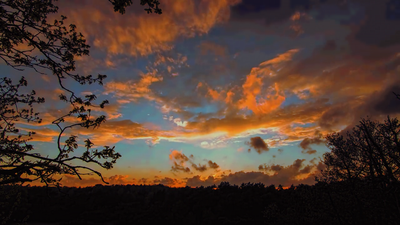}\hspace{\imhspace}
    \adjincludegraphics[width=\imwidth]{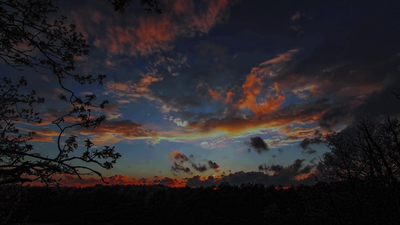}\hspace{\imhspace}
    \adjincludegraphics[width=\imwidth]{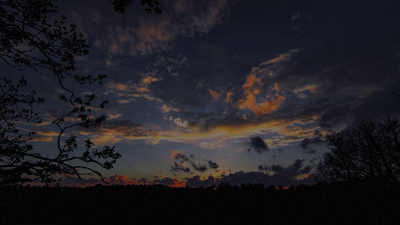} \\
   	\vspace{3\imvspace}
	\raisebox{0em}{\rotatebox{90}{Direct transfer}}\hspace{0.5em}
	\adjincludegraphics[width=\imwidth]{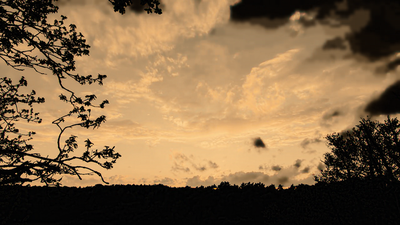}\hspace{\imhspace}
	\adjincludegraphics[width=\imwidth]{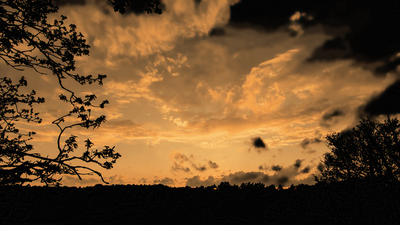}\hspace{\imhspace}
	\adjincludegraphics[width=\imwidth]{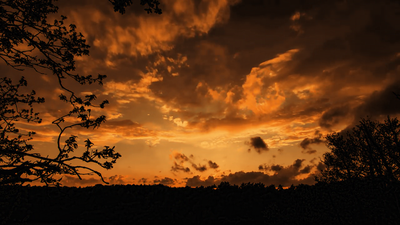}\hspace{\imhspace}
	\adjincludegraphics[width=\imwidth]{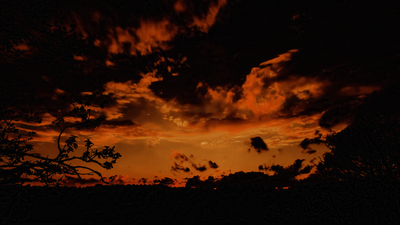}\hspace{\imhspace}
	\adjincludegraphics[width=\imwidth]{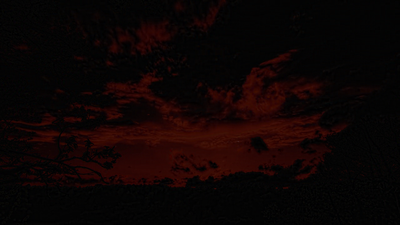} \\
	
	\caption{1st row: ground truth, the \emph{country1} dataset.
	Rows 2-4: color transfer of each interpolated histograms in \autoref{fig:quincy1_sel2F0016I500_hists}.
	When comparing with the ground truth, we see that our method recreates sunset-like colors, as opposed to the two other methods.
	Row 5: direct transfer of each frame of the \emph{seldovia2} dataset on the first frame of the \emph{country1} dataset.
	Our method is able to preserve the original colors of the day image, contrary to a direct transfer.}
	\label{fig:quincy1_sel2F0016I500_images}
\end{figure*}

\section{Discussion}
\label{sec:discussion}

The problem we tackle is ill-posed and in general there is no way to find information where mass does not travel.
Nevertheless, our regularization of the problem reduces the number of local minima and reduces the non-convexity by imposing spatially smooth metric weights, which also avoids over-fitting.

The parameterization we chose for the metric is limited in the sense that it only includes diffusion along the grid axes, which leads to low-precision approximations of the heat kernel on small domains.
This approximation affects the quality of the displacement interpolations, as seen in \autoref{fig:eval_interp}.
It leads to a trade-off when choosing $\varepsilon$, between the smoothness of the interpolations, the regularity with which they are spread out spatially, and the computational limits it involves (as $\sclr{S}$ increases).
Moreover, as pointed out in \cite[Appendix A]{crane_geodesics_2013}, low values for the parameter $\varepsilon$ yield a distance that is closer to the graph distance (number of edges), than to the geodesic distance.
This means that as $\varepsilon$ decreases, the edge weights have diminishing influence.

Although we managed to develop a tractable framework, as compared for instance to using a dense storage of the cost matrix, this algorithm remains computationally expensive for histograms with more than 10 000 points (see \autoref{tab:timings}).

\section{Conclusion}
\label{sec:conclusion}

We have proposed a new method to learn the ground metric of optimal transport, as a geodesic distance on the graph supporting the data.
We learn from observations of a mass displacement and aim to reconstruct them using displacement interpolations.
We were able to turn a challenging task in terms of time and memory complexity into a tractable framework, using diffusion-based distance computations, regularized Wasserstein barycenters, and automatic differentiation.
We demonstrated our method on toy examples, as well as on a color transfer application, where we learn the evolution of colors during a sunset, and use it to create a new sunset sequence.
We finally discussed the limitations of the proposed method: our parametrization of the metric might be too simple, which limits the precision of the geodesic distances approximation. 
In turn, this impacts the interpolation, and adds trade-offs between having sharp interpolants, equally-spaced interpolants and the computational effort required for these.

\paragraph{Future work.}

For regular domains such as images and surfaces, it is possible to use a more precise approximation of a Riemannian metric as a field of tensors instead of a graph, as done for instance in~\cite{mirebeau_automatic_2017}, which in turn can be combined with triangulated meshes.
Multi-resolution strategies can also be integrated into our pipeline to accelerate the linear system resolution and Sinkhorn algorithm (as proposed by \cite{gerber_multiscale_2017}).
\new{
In this paper, we conducted only preliminary experiments for learning from multiple sequences.
The results seem promising, and we could imagine replacing our quasi-Newton optimizer by a stochastic gradient descent to potentially accelerate the learning.
Unbalanced optimal transport \cite{chizat_scaling_2016} could also be valuable to account for mass creation and elimination during transport, which is crucial for some applications in chemistry or biology.
}

\begin{acknowledgements}
This work was supported by the French National Research Agency (ANR) through the ROOT research grant (ANR-16-CE23-0009).
The work of G. Peyr\'e was supported by the European Research Council (ERC project NORIA) and by the French government under management of Agence Nationale de la Recherche as part of the ``Investissements d'avenir'' program, reference ANR19-P3IA-0001 (PRAIRIE 3IA Institute).
\end{acknowledgements}

\bibliographystyle{spmpscinat}
\bibliography{OTML}

\begin{thebibliography}{57}
\providecommand{\natexlab}[1]{#1}
\providecommand{\url}[1]{#1}
\providecommand{\urlprefix}{URL }
\expandafter\ifx\csname urlstyle\endcsname\relax
  \providecommand{\doi}[1]{DOI~\discretionary{}{}{}#1}\else
  \providecommand{\doi}{DOI~\discretionary{}{}{}\begingroup
  \urlstyle{rm}\Url}\fi

\bibitem[{Agueh and Carlier(2011)}]{agueh_barycenters_2011}
Agueh, M., Carlier, G.: Barycenters in the {{Wasserstein Space}}.
\newblock SIAM Journal on Mathematical Analysis \textbf{43}(2), 904--924
  (2011).
\newblock \doi{10.1137/100805741}

\bibitem[{Altschuler et~al.(2018)Altschuler, Bach, Rudi, and
  Weed}]{altschuler_massively_2018}
Altschuler, J., Bach, F., Rudi, A., Weed, J.: Massively scalable {{Sinkhorn}}
  distances via the {{Nystr\"om}} method.
\newblock arXiv:1812.05189 [cs, math, stat]  (2018)

\bibitem[{Altschuler et~al.(2017)Altschuler, Weed, and
  Rigollet}]{altschuler_near-linear_2017}
Altschuler, J., Weed, J., Rigollet, P.: Near-linear time approximation
  algorithms for optimal transport via {{Sinkhorn}} iteration.
\newblock arXiv preprint arXiv:1705.09634  (2017)

\bibitem[{Angenent et~al.(2003)Angenent, Haker, and
  Tannenbaum}]{angenent_minimizing_2003}
Angenent, S., Haker, S., Tannenbaum, A.: Minimizing {{Flows}} for the
  {{Monge}}-{{Kantorovich Problem}}.
\newblock SIAM Journal on Mathematical Analysis \textbf{35}(1), 61--97 (2003).
\newblock \doi{10.1137/S0036141002410927}

\bibitem[{Bellet et~al.(2015)Bellet, Habrard, and Sebban}]{bellet_metric_2015}
Bellet, A., Habrard, A., Sebban, M.: Metric Learning.
\newblock Synthesis Digital Library of Engineering and Computer Science. {San
  Rafael, California (1537 Fourth Street, San Rafael, CA 94901 USA): Morgan \&
  Claypool} (2015)

\bibitem[{Benamou et~al.(2015)Benamou, Carlier, Cuturi, Nenna, and
  Peyr{\'e}}]{benamou_iterative_2015}
Benamou, J.D., Carlier, G., Cuturi, M., Nenna, L., Peyr{\'e}, G.: Iterative
  {{Bregman}} projections for regularized transportation problems.
\newblock SIAM Journal on Scientific Computing \textbf{37}(2), A1111--A1138
  (2015)

\bibitem[{Benmansour et~al.(2010)Benmansour, Carlier, Peyr{\'e}, and
  Santambrogio}]{benmansour_derivatives_2010}
Benmansour, F., Carlier, G., Peyr{\'e}, G., Santambrogio, F.: Derivatives with
  {{Respect}} to {{Metrics}} and {{Applications}}: {{Subgradient Marching
  Algorithm}}.
\newblock Numerische Mathematik \textbf{116}(3), 357--381 (2010).
\newblock \doi{10.1007/s00211-010-0305-8}

\bibitem[{Bonneel et~al.(2016)Bonneel, Peyr{\'e}, and
  Cuturi}]{bonneel_wasserstein_2016}
Bonneel, N., Peyr{\'e}, G., Cuturi, M.: Wasserstein barycentric coordinates:
  Histogram regression using optimal transport.
\newblock ACM Transactions on Graphics \textbf{35}(4), 1--10 (2016).
\newblock \doi{10.1145/2897824.2925918}

\bibitem[{Brickell et~al.(2008)Brickell, Dhillon, Sra, and
  Tropp}]{brickell_metric_2008}
Brickell, J., Dhillon, I.S., Sra, S., Tropp, J.A.: The {{Metric Nearness
  Problem}}.
\newblock SIAM Journal on Matrix Analysis and Applications \textbf{30}(1),
  375--396 (2008).
\newblock \doi{10.1137/060653391}

\bibitem[{Buttazzo et~al.(2004)Buttazzo, Davini, Fragal{\`a}, and
  Maci{\`a}}]{buttazzo_optimal_2004}
Buttazzo, G., Davini, A., Fragal{\`a}, I., Maci{\`a}, F.: Optimal
  {{Riemannian}} distances preventing mass transfer.
\newblock Journal f\"ur die reine und angewandte Mathematik (Crelles Journal)
  \textbf{2004}(575) (2004).
\newblock \doi{10.1515/crll.2004.077}

\bibitem[{Chechik et~al.(2009)Chechik, Shalit, Sharma, and
  Bengio}]{chechik_online_2009}
Chechik, G., Shalit, U., Sharma, V., Bengio, S.: An {{Online Algorithm}} for
  {{Large Scale Image Similarity Learning}}.
\newblock Advances in Neural Information Processing Systems p.~9 (2009)

\bibitem[{Chizat et~al.(2016)Chizat, Peyr{\'e}, Schmitzer, and
  Vialard}]{chizat_scaling_2016}
Chizat, L., Peyr{\'e}, G., Schmitzer, B., Vialard, F.X.: Scaling algorithms for
  unbalanced transport problems.
\newblock arXiv preprint arXiv:1607.05816  (2016)

\bibitem[{Chopra et~al.(2005)Chopra, Hadsell, and LeCun}]{chopra_learning_2005}
Chopra, S., Hadsell, R., LeCun, Y.: Learning a {{Similarity Metric
  Discriminatively}}, with {{Application}} to {{Face Verification}}.
\newblock In: 2005 {{IEEE Computer Society Conference}} on {{Computer Vision}}
  and {{Pattern Recognition}} ({{CVPR}}'05), vol.~1, pp. 539--546. {IEEE}, {San
  Diego, CA, USA} (2005).
\newblock \doi{10.1109/CVPR.2005.202}

\bibitem[{Courty et~al.(2016)Courty, Flamary, Tuia, and
  Rakotomamonjy}]{courty_optimal_2016}
Courty, N., Flamary, R., Tuia, D., Rakotomamonjy, A.: Optimal {{Transport}} for
  {{Domain Adaptation}}.
\newblock IEEE Transactions on Pattern Analysis and Machine Intelligence
  \textbf{39}(9), 1853--1865 (2016).
\newblock \doi{10.1109/TPAMI.2016.2615921}

\bibitem[{Crane et~al.(2013)Crane, Weischedel, and
  Wardetzky}]{crane_geodesics_2013}
Crane, K., Weischedel, C., Wardetzky, M.: Geodesics in heat: {{A}} new approach
  to computing distance based on heat flow.
\newblock ACM Transactions on Graphics (TOG) \textbf{32}(5), 152 (2013)

\bibitem[{Cuturi(2013)}]{cuturi_sinkhorn_2013}
Cuturi, M.: Sinkhorn distances: {{Lightspeed}} computation of optimal
  transport.
\newblock In: Advances in {{Neural Information Processing Systems}}, pp.
  2292--2300 (2013)

\bibitem[{Cuturi and Avis(2014)}]{cuturi_ground_2014}
Cuturi, M., Avis, D.: Ground metric learning.
\newblock Journal of Machine Learning Research \textbf{15}(1), 533--564 (2014)

\bibitem[{Cuturi and Doucet(2014)}]{cuturi_fast_2014}
Cuturi, M., Doucet, A.: Fast computation of {{Wasserstein}} barycenters.
\newblock In: International {{Conference}} on {{Machine Learning}}, pp.
  685--693 (2014)

\bibitem[{Dognin et~al.(2019)Dognin, Melnyk, Mroueh, Ross, Santos, and
  Sercu}]{dognin_wasserstein_2019}
Dognin, P., Melnyk, I., Mroueh, Y., Ross, J., Santos, C.D., Sercu, T.:
  Wasserstein {{Barycenter Model Ensembling}}.
\newblock arXiv:1902.04999 [cs, stat]  (2019)

\bibitem[{Dupuy et~al.(2016)Dupuy, Galichon, and Sun}]{dupuy_estimating_2016}
Dupuy, A., Galichon, A., Sun, Y.: Estimating matching affinity matrix under
  low-rank constraints.
\newblock arXiv:1612.09585 [stat]  (2016)

\bibitem[{Dvurechensky et~al.(2018)Dvurechensky, Gasnikov, and
  Kroshnin}]{dvurechensky_computational_2018}
Dvurechensky, P., Gasnikov, A., Kroshnin, A.: Computational {{Optimal
  Transport}}: {{Complexity}} by {{Accelerated Gradient Descent Is Better
  Than}} by {{Sinkhorn}}'s {{Algorithm}}.
\newblock arXiv:1802.04367 [cs, math]  (2018)

\bibitem[{Frogner et~al.(2015)Frogner, Zhang, Mobahi, Araya, and
  Poggio}]{frogner_learning_2015}
Frogner, C., Zhang, C., Mobahi, H., Araya, M., Poggio, T.A.: Learning with a
  {{Wasserstein Loss}}.
\newblock Advances in Neural Information Processing Systems p.~9 (2015)

\bibitem[{Genevay et~al.(2017)Genevay, Peyr{\'e}, and
  Cuturi}]{genevay_learning_2017}
Genevay, A., Peyr{\'e}, G., Cuturi, M.: Learning {{Generative Models}} with
  {{Sinkhorn Divergences}}.
\newblock arXiv:1706.00292 [stat]  (2017)

\bibitem[{Gerber and Maggioni(2017)}]{gerber_multiscale_2017}
Gerber, S., Maggioni, M.: Multiscale strategies for computing optimal
  transport.
\newblock arXiv preprint arXiv:1708.02469  (2017)

\bibitem[{Griewank(2012)}]{griewank_who_2012}
Griewank, A.: Who {{Invented}} the {{Reverse Mode}} of {{Differentiation}}?
\newblock Documenta Mathematica p.~12 (2012)

\bibitem[{Griewank and Walther(2008)}]{griewank_evaluating_2008}
Griewank, A., Walther, A.: Evaluating {{Derivatives}}.
\newblock Other {{Titles}} in {{Applied Mathematics}}. {Society for Industrial
  and Applied Mathematics} (2008).
\newblock \doi{10.1137/1.9780898717761}

\bibitem[{Huang et~al.(2016)Huang, Guo, Kusner, Sun, Sha, and
  Weinberger}]{huang_supervised_2016}
Huang, G., Guo, C., Kusner, M.J., Sun, Y., Sha, F., Weinberger, K.Q.:
  Supervised word mover's distance.
\newblock In: Advances in {{Neural Information Processing Systems}}, pp.
  4862--4870 (2016)

\bibitem[{Kedem et~al.(2012)Kedem, Tyree, Sha, Lanckriet, and
  Weinberger}]{kedem_non-linear_2012}
Kedem, D., Tyree, S., Sha, F., Lanckriet, G.R., Weinberger, K.Q.: Non-linear
  {{Metric Learning}}.
\newblock Neural Information Processing Systems (NIPS) p.~9 (2012)

\bibitem[{Kulis(2013)}]{kulis_metric_2013}
Kulis, B.: Metric {{Learning}}: {{A Survey}}.
\newblock Foundations and Trends\textregistered{} in Machine Learning
  \textbf{5}(4), 287--364 (2013).
\newblock \doi{10.1561/2200000019}

\bibitem[{L{\'e}vy(2015)}]{levy_numerical_2015}
L{\'e}vy, B.: A {{Numerical Algorithm}} for {{L2 Semi}}-{{Discrete Optimal
  Transport}} in {{3D}}.
\newblock ESAIM: Mathematical Modelling and Numerical Analysis \textbf{49}(6),
  1693--1715 (2015).
\newblock \doi{10.1051/m2an/2015055}

\bibitem[{Li et~al.(2019)Li, Ye, Zhou, and Zha}]{li_learning_2019}
Li, R., Ye, X., Zhou, H., Zha, H.: Learning to {{Match}} via {{Inverse Optimal
  Transport}}.
\newblock Journal of Machine Learning Research p.~37 (2019)

\bibitem[{MacAdam(1942)}]{macadam_visual_1942}
MacAdam, D.L.: Visual {{Sensitivities}} to {{Color Differences}} in
  {{Daylight}}.
\newblock Journal of the Optical Society of America \textbf{32}(5), 247 (1942).
\newblock \doi{10.1364/JOSA.32.000247}

\bibitem[{McCann(1997)}]{mccann_convexity_1997}
McCann, R.J.: A {{Convexity Principle}} for {{Interacting Gases}}.
\newblock Advances in Mathematics \textbf{128}(1), 153--179 (1997).
\newblock \doi{10.1006/aima.1997.1634}

\bibitem[{Mirebeau and Dreo(2017)}]{mirebeau_automatic_2017}
Mirebeau, J.M., Dreo, J.: Automatic differentiation of non-holonomic fast
  marching for computing most threatening trajectories under sensors
  surveillance.
\newblock arXiv:1704.03782 [math]  (2017)

\bibitem[{Papadakis et~al.(2014)Papadakis, Peyr{\'e}, and
  Oudet}]{papadakis_optimal_2014}
Papadakis, N., Peyr{\'e}, G., Oudet, E.: Optimal {{Transport}} with {{Proximal
  Splitting}}.
\newblock SIAM Journal on Imaging Sciences \textbf{7}(1), 212--238 (2014).
\newblock \doi{10.1137/130920058}

\bibitem[{Paszke et~al.(2017)Paszke, Gross, Chintala, Chanan, Yang, DeVito,
  Lin, Desmaison, Antiga, and Lerer}]{paszke_automatic_2017}
Paszke, A., Gross, S., Chintala, S., Chanan, G., Yang, E., DeVito, Z., Lin, Z.,
  Desmaison, A., Antiga, L., Lerer, A.: Automatic differentiation in
  {{PyTorch}} p.~4 (2017)

\bibitem[{Pele and {Ben-Aliz}(2016)}]{pele_interpolated_2016}
Pele, O., {Ben-Aliz}, Y.: Interpolated {{Discretized Embedding}} of {{Single
  Vectors}} and {{Vector Pairs}} for {{Classification}}, {{Metric Learning}}
  and {{Distance Approximation}}.
\newblock arXiv:1608.02484 [cs]  (2016)

\bibitem[{Peyr{\'e} and Cuturi(2018)}]{peyre_computational_2018}
Peyr{\'e}, G., Cuturi, M.: Computational {{Optimal Transport}}.
\newblock {Now Publishers, Inc.} (2018)

\bibitem[{Rubner et~al.(2000)Rubner, Tomasi, and Guibas}]{rubner_earth_2000}
Rubner, Y., Tomasi, C., Guibas, L.J.: The {{Earth Mover}}'s {{Distance}} as a
  {{Metric}} for {{Image Retrieval}}.
\newblock International Journal of Computer Vision p.~23 (2000)

\bibitem[{Sandler and Lindenbaum(2011)}]{sandler_nonnegative_2011}
Sandler, R., Lindenbaum, M.: Nonnegative {{Matrix Factorization}} with {{Earth
  Mover}}'s {{Distance Metric}} for {{Image Analysis}}.
\newblock IEEE Transactions on Pattern Analysis and Machine Intelligence
  \textbf{33}(8), 1590--1602 (2011).
\newblock \doi{10.1109/TPAMI.2011.18}

\bibitem[{Santambrogio(2015)}]{santambrogio_optimal_2015}
Santambrogio, F.: Optimal Transport for Applied Mathematicians.
\newblock {Springer} (2015)

\bibitem[{Schiebinger et~al.(2019)Schiebinger, Shu, Tabaka, Cleary,
  Subramanian, Solomon, Gould, Liu, Lin, Berube, Lee, Chen, Brumbaugh,
  Rigollet, Hochedlinger, Jaenisch, Regev, and
  Lander}]{schiebinger_optimal-transport_2019}
Schiebinger, G., Shu, J., Tabaka, M., Cleary, B., Subramanian, V., Solomon, A.,
  Gould, J., Liu, S., Lin, S., Berube, P., Lee, L., Chen, J., Brumbaugh, J.,
  Rigollet, P., Hochedlinger, K., Jaenisch, R., Regev, A., Lander, E.S.:
  Optimal-{{Transport Analysis}} of {{Single}}-{{Cell Gene Expression
  Identifies Developmental Trajectories}} in {{Reprogramming}}.
\newblock Cell \textbf{176}(4), 928--943.e22 (2019).
\newblock \doi{10.1016/j.cell.2019.01.006}

\bibitem[{Schmitz et~al.(2018)Schmitz, Heitz, Bonneel, Ngol{\`e}~Mboula,
  Coeurjolly, Cuturi, Peyr{\'e}, and Starck}]{schmitz_wasserstein_2018}
Schmitz, M.A., Heitz, M., Bonneel, N., Ngol{\`e}~Mboula, F.M., Coeurjolly, D.,
  Cuturi, M., Peyr{\'e}, G., Starck, J.L.: Wasserstein {{Dictionary Learning}}:
  {{Optimal Transport}}-based unsupervised non-linear dictionary learning.
\newblock SIAM Journal on Imaging Sciences  (2018)

\bibitem[{Simou and Frossard(2018)}]{simou_graph_2018}
Simou, E., Frossard, P.: Graph {{Signal Representation}} with {{Wasserstein
  Barycenters}}.
\newblock arXiv:1812.05517 [eess]  (2018)

\bibitem[{Solomon et~al.(2015)Solomon, {de Goes}, Peyr{\'e}, Cuturi, Butscher,
  Nguyen, Du, and Guibas}]{solomon_convolutional_2015}
Solomon, J., {de Goes}, F., Peyr{\'e}, G., Cuturi, M., Butscher, A., Nguyen,
  A., Du, T., Guibas, L.: Convolutional {{Wasserstein Distances}}: {{Efficient
  Optimal Transportation}} on {{Geometric Domains}}.
\newblock ACM Trans. Graph. \textbf{34}(4), 66:1--66:11 (2015).
\newblock \doi{10.1145/2766963}

\bibitem[{Stuart and Wolfram(2019)}]{stuart_inverse_2019}
Stuart, A.M., Wolfram, M.T.: Inverse optimal transport.
\newblock arXiv:1905.03950 [math, stat]  (2019)

\bibitem[{Torresani and Lee(2007)}]{torresani_large_2007}
Torresani, L., Lee, K.c.: Large {{Margin Component Analysis}}.
\newblock Advances in neural information processing systems p.~8 (2007)

\bibitem[{Varadhan(1967)}]{varadhan_behavior_1967}
Varadhan, S.R.S.: On the behavior of the fundamental solution of the heat
  equation with variable coefficients.
\newblock Communications on Pure and Applied Mathematics \textbf{20}(2),
  431--455 (1967)

\bibitem[{Wang and Guibas(2012)}]{wang_supervised_2012}
Wang, F., Guibas, L.J.: Supervised {{Earth Mover}}'s {{Distance Learning}} and
  {{Its Computer Vision Applications}}.
\newblock In: Computer {{Vision}} \textendash{} {{ECCV}} 2012, vol. 7572, pp.
  442--455. {Springer Berlin Heidelberg}, {Berlin, Heidelberg} (2012).
\newblock \doi{10.1007/978-3-642-33718-5_32}

\bibitem[{Wang et~al.(2011)Wang, Do, Woznica, and Kalousis}]{wang_metric_2011}
Wang, J., Do, H.T., Woznica, A., Kalousis, A.: Metric {{Learning}} with
  {{Multiple Kernels}}.
\newblock Advances in Neural Information Processing Systems p.~9 (2011)

\bibitem[{Weinberger et~al.(2006)Weinberger, Blitzer, and
  Saul}]{weinberger_distance_2006}
Weinberger, K.Q., Blitzer, J., Saul, L.K.: Distance {{Metric Learning}} for
  {{Large Margin Nearest Neighbor Classification}}.
\newblock Advances in neural information processing systems p.~8 (2006)

\bibitem[{Weinberger and Saul(2008)}]{weinberger_fast_2008}
Weinberger, K.Q., Saul, L.K.: Fast solvers and efficient implementations for
  distance metric learning.
\newblock In: Proceedings of the 25th International Conference on {{Machine}}
  Learning - {{ICML}} '08, pp. 1160--1167. {ACM Press}, {Helsinki, Finland}
  (2008).
\newblock \doi{10.1145/1390156.1390302}

\bibitem[{Xing et~al.(2003)Xing, Jordan, Russell, and Ng}]{xing_distance_2003}
Xing, E.P., Jordan, M.I., Russell, S.J., Ng, A.Y.: Distance {{Metric Learning}}
  with {{Application}} to {{Clustering}} with {{Side}}-{{Information}}.
\newblock In: Advances in Neural Information Processing Systems, p.~8 (2003)

\bibitem[{Xu et~al.(2018)Xu, Luo, Deng, and Huang}]{xu_multi-level_2018}
Xu, J., Luo, L., Deng, C., Huang, H.: Multi-{{Level Metric Learning}} via
  {{Smoothed Wasserstein Distance}}.
\newblock In: Proceedings of the {{Twenty}}-{{Seventh International Joint
  Conference}} on {{Artificial Intelligence}}, pp. 2919--2925. {International
  Joint Conferences on Artificial Intelligence Organization}, {Stockholm,
  Sweden} (2018).
\newblock \doi{10.24963/ijcai.2018/405}

\bibitem[{Yang and Cohen(2016)}]{yang_geodesic_2016}
Yang, F., Cohen, L.D.: Geodesic {{Distance}} and {{Curves Through Isotropic}}
  and {{Anisotropic Heat Equations}} on {{Images}} and {{Surfaces}}.
\newblock Journal of Mathematical Imaging and Vision \textbf{55}(2), 210--228
  (2016).
\newblock \doi{10.1007/s10851-015-0621-9}

\bibitem[{Yang et~al.(2015)Yang, Xu, Chen, Zheng, and
  Liu}]{yang_chi-squared_2015}
Yang, W., Xu, L., Chen, X., Zheng, F., Liu, Y.: Chi-{{Squared Distance Metric
  Learning}} for {{Histogram Data}}.
\newblock Mathematical Problems in Engineering \textbf{2015}, 1--12 (2015).
\newblock \doi{10.1155/2015/352849}

\bibitem[{Zen et~al.(2014)Zen, Ricci, and Sebe}]{zen_simultaneous_2014}
Zen, G., Ricci, E., Sebe, N.: Simultaneous {{Ground Metric Learning}} and
  {{Matrix Factorization}} with {{Earth Mover}}'s {{Distance}}.
\newblock In: 2014 22nd {{International Conference}} on {{Pattern
  Recognition}}, pp. 3690--3695 (2014).
\newblock \doi{10.1109/ICPR.2014.634}

\end{thebibliography}

\appendix

\section{Elements of proof for Proposition 1}
\label{apx:kernelgrad}

The mapping $\phi_1 : \vec{w} \in \mathbb{R}^\sclr{K} \to \mat{M} \in \mathbb{R}^{\sclr{N}^2}$ admits as adjoint Jacobian:
\begin{equation}
	\left[\partial \phi_1(\vec{w})\right]^T(\mat{X})_{i,j} = -\frac{\varepsilon}{4\sclr{S}}\left(\mat{X}_{i,j}+\mat{X}_{j,i}-\mat{X}_{i,i}-\mat{X}_{j,j}\right).
\end{equation}

The mapping $\phi_2 : \mat{M} \in \mathbb{R}^{\sclr{N}^2} \to \mat{U}=\mat{M}^{-1}\in \mathbb{R}^{\sclr{N}^2}$ admits as adjoint Jacobian:
\begin{equation}
\left[\partial \phi_2(\mat{M})\right]^T(\mat{X}) = -\mat{M}^{-1}\mat{X}\mat{M}^{-1}.
\end{equation}

The mapping $\phi_3 : \mat{U} \in \mathbb{R}^{\sclr{N}^2} \to \mat{V}=\mat{U}^{\sclr{S}} \in \mathbb{R}^{\sclr{N}^2}$ admits as adjoint Jacobian:
\begin{equation}
\left[\partial \phi_3(\mat{U})\right]^T(\mat{X}) = \sum_{l=0}^{\sclr{S}-1}\mat{U}^{l}\mat{X}\mat{U}^{S-l-1}.
\end{equation}

The mapping $\phi_4 : \mat{V} \in \mathbb{R}^{\sclr{N}^2} \to \vec{y}=\mat{V}\vec{v} \in \mathbb{R}^{\sclr{N}}$ admits as adjoint Jacobian:
\begin{equation}
\left[\partial \phi_4(\mat{V})\right]^T(\vec{x}) = \vec{x}\vec{v}^T.
\end{equation}

Since $\Phi = \phi_4 \circ \phi_3 \circ \phi_2 \circ \phi_1$, we compose the adjoint Jacobians in the reverse order as follows:
\begin{equation}
	\left[\partial \Phi(\vec{w})\right]^T(\mat{g})_{i,j} = \left[\partial \phi_1\right]^T\left[\partial \phi_2\right]^T\left[\partial \phi_3\right]^T\left[\partial \phi_4\right]^T(\vec{g})_{i,j},
\end{equation}
to finally obtain:
\begin{align}
&[\partial \Phi(\vec{w})]^T(\vec{g})_{i,j}
= 
\frac{\varepsilon}{4\sclr{S}}\sum_{\ell=0}^{\sclr{S}-1} \left(\vec{g}_i^\ell - \vec{g}_j^\ell \right) \left(\vec{v}_i^\ell - \vec{v}_j^\ell \right), \\
&\text{where}\quad
\left\{
\begin{array}{l}
\vec{g}^\ell \eqdef \mat{M}^{\ell-\sclr{S}}\vec{g} \\
\vec{v}^\ell \eqdef \mat{M}^{-\ell-1} \vec{v}.
\end{array}
\right.
\end{align}

\end{document}